\documentclass[11pt,letter]{article}

\usepackage[utf8]{inputenc}
\usepackage[T1]{fontenc}

\usepackage{fullpage} %

\usepackage[small]{caption}

\usepackage{amsmath}
\usepackage{amsfonts}
\usepackage{amssymb}
\usepackage{amsthm}
\usepackage{amstext}
\usepackage{thmtools}

\usepackage{csquotes}
\usepackage{xcolor}
\usepackage{graphicx}

\usepackage{cite} %
\usepackage{xspace}
\usepackage{booktabs, mathtools} 
\usepackage{adjustbox}
\usepackage{paralist} 
\usepackage{caption}
\usepackage{subcaption}
\usepackage[shortlabels]{enumitem}

\usepackage{letltxmacro}
\newlist{thmlist}{enumerate}{1}
\setlist[thmlist]{label=(\roman{thmlisti}), ref=\thetheorem(\roman{thmlisti}),noitemsep}

\usepackage{url}
\usepackage[hidelinks]{hyperref}
\usepackage[capitalize]{cleveref}
\Crefname{lemma}{Lemma}{Lemmata}

\usepackage{libertine}
\usepackage[libertine]{newtxmath}
\usepackage{microtype}

\newtheorem{theorem}{Theorem}

\newtheorem{lemma}[theorem]{Lemma}

\newtheorem{observation}[theorem]{Observation}

\addtotheorempostheadhook[lemma]{\crefalias{thmlisti}{lemma}}

\newcommand{\ldc}{\textsc{LambdaDC}\xspace}
\newcommand{\dc}{\textsc{DoubleCoverage}\xspace}
\newcommand{\ftpdc}{\textsc{FtP}\&\textsc{DC}\xspace}
\newcommand{\ftp}{{\normalfont\textsc{FtP}}\xspace}

\newcommand{\alg}{{\normalfont\textsc{Alg}}}
\newcommand{\opt}{{\normalfont\textsc{Opt}}\xspace}

\newcommand{\A}{\ensuremath{\mathcal{A}}\xspace}
\newcommand{\B}{\mathcal{B}}
\newcommand{\bigO}{\mathcal{O}}
\newcommand{\E}{\mathbb{E}}

\newcommand{\Hyp}{\mathcal{H}}
\newcommand{\pred}{\tau}
\newcommand{\seq}{\sigma}

\newcommand{\D}{\mathcal{D}}

\newcommand{\cons}{\alpha}
\newcommand{\rob}{\beta}
\newcommand{\local}{{locally-consistent}\xspace}
\newcommand{\memoryless}{{memory-constrained}\xspace}
\newcommand{\defn}[1]{\textbf{{#1}}\xspace}

\DeclarePairedDelimiter{\abs}{\lvert}{\rvert}
\DeclarePairedDelimiter\ceil{\lceil}{\rceil}
\DeclarePairedDelimiter\floor{\lfloor}{\rfloor}

\usepackage[colorinlistoftodos,prependcaption,textsize=scriptsize]{todonotes}

\usepackage{tikz}
\usetikzlibrary{decorations.pathreplacing}
\usetikzlibrary{calc}
\tikzset{
       server/.style = {circle, fill, minimum size=#1, inner sep=0pt, outer sep=0pt},
       server/.default = 3mm,
       active_server/.style = {circle, fill=active, draw=active, minimum size=#1, inner sep=0pt, outer sep=0pt},
       active_server/.default = 3mm,
       predicted_server/.style={circle, fill=black, draw=algLightGreen, very thick, minimum size=#1, inner sep=0pt, outer sep=0pt},
       predicted_server/.default = 3mm,
       active_predicted_server/.style={circle, fill=active, draw=algLightGreen, very thick, minimum size=#1, inner sep=0pt, outer sep=0pt},
       active_predicted_server/.default = 3mm,
       req/.style = {circle, fill=reqRed, minimum size=#1, inner sep=0pt, outer sep=0pt},
       req/.default = 2.25mm,
       sline/.style = {line width=2.1pt},
       edge/.style = {line width=1.5pt},
       helperline/.style={dashed},
       algMove/.style = {ultra thick, ->, algBlue},     
       relocation/.style = {ultra thick, ->, algPurple},
       dcMove/.style = {ultra thick, ->, algGreen},
       optMove/.style = {ultra thick, ->, algOrange},
       predMove/.style = {ultra thick, ->, algLightGreen},
}

\definecolor{algBlue}{HTML}{3c63c6}
\definecolor{algOrange}{HTML}{ee9f14}
\definecolor{algRed}{HTML}{b9372a}
\definecolor{algGreen}{HTML}{26a68f}
\definecolor{algLightGreen}{HTML}{40bd1e}
\definecolor{algPurple}{HTML}{58118f}
\definecolor{algBrown}{HTML}{9c5f0b}
\definecolor{reqRed}{HTML}{96183a}
\definecolor{active}{HTML}{edbe00}

\graphicspath{{./img/}}%

\title{Double Coverage with Machine-Learned Advice}

\author{
Alexander Lindermayr\thanks{Faculty of Mathematics and Computer Science, University of Bremen, Germany. \emph{\{linderal,nmegow\}@uni-bremen.de}} 
\and
Nicole Megow\footnotemark[1] 
\and
Bertrand Simon\thanks{IN2P3 Computing Center, CNRS, Villeurbanne, France. \emph{bertrand.simon@cc.in2p3.fr}}
}

\date{}

\begin{document}

\maketitle

\begin{abstract}
We study the fundamental online $k$-server problem in a learning-augmented setting. While in the traditional online model, an algorithm has no information about the request sequence, we assume that there is given some advice (e.g.~machine-learned predictions) on an algorithm's decision. There is, however, no guarantee on the quality of the prediction and it might be far from being correct.

Our main result is a learning-augmented variation of the well-known Double Coverage algorithm for $k$-server on the line~(Chrobak et al., SIDMA 1991) in which we integrate predictions as well as our trust into their quality.
We give an error-dependent competitive ratio, which is a function of a user-defined confidence %
parameter, and which  interpolates smoothly between an optimal consistency, the performance in case that all predictions are correct, and the best-possible robustness regardless of the prediction quality. When given good predictions, we improve upon known lower bounds for online algorithms without advice. 
We further show that our algorithm achieves for any $k$ an almost %
optimal consistency-robustness tradeoff, within a class of deterministic algorithms respecting \emph{local} and \emph{memoryless} properties. %

Our algorithm outperforms a previously proposed (more general) 
learning-augmented algorithm. %
It is remarkable that the previous algorithm crucially exploits memory, whereas our algorithm is \emph{memoryless}. %
Finally, we demonstrate in experiments the practicability and the superior performance of our algorithm on real-world data.

\end{abstract}

\thispagestyle{empty}

\section{Introduction}\label{sec:intro}

The $k$-server problem is one of the most fundamental online optimization problems. Manasse et al.~\cite{manasse1990server,ManasseMS88} introduced it in 1988 as a generalization of other online problems, such as the prominent paging problem, and since then, it has been a corner stone for developing new models and techniques. We follow this line and investigate the $k$-server problem in the recently evolving framework of learning-augmented online computation.

We consider the~{\em $k$-server problem on the line}, in which there are given~$k$ distinct servers~$s_1,\ldots,s_k$ located at initial positions on the real line. A sequence of requests~$r_1,\ldots,r_n \in \mathbb{R}$ is revealed online one-by-one, that is, an algorithm only knows the current (unserved) request, serves it and only then sees the next request; it has no knowledge about future requests. To serve a request, (at least) one of the servers has to be moved to the requested point. The cost of serving a request is defined as the distance traveled by the server(s). The task is to give an online strategy of minimum total cost for serving a request sequence.

In standard competitive analysis, an online algorithm~$\A$ %
is called~$\mu$-\emph{competitive} if for every instance~$I$, there is some constant~$c$ depending only on the initial configuration %
such that~$\A(I) \leq \mu \cdot \opt(I) + c $, where~$\A(I)$ denotes the cost of~$\A$ on~$I$ whereas~$\opt(I)$ is the cost of an optimal solution that can be obtained when having full information about $I$ in~advance. %

Manasse et al.~\cite{manasse1990server} gave a strong lower bound which rules out any  deterministic online algorithm with a competitive ratio better than~$k$. They also stated the famous \emph{k-server conjecture} in which they conjecture that there is a $k$-competitive online algorithm for the $k$-server problem in any metric space and for any $k$. The conjecture has been proven to be true for special metric spaces such as the line~\cite{chrobak1991dc}, %
considered in this paper, the uniform metric space (paging problem)~\cite{SleatorT85} and tree metrics~\cite{ChrobakL91}.  
For the $k$-server problem on the line, Chrobak et al.~\cite{chrobak1991dc} devised the \dc algorithm and proved a best possible competitive ratio~$k$. %
For a given request, \dc moves the (at most) two adjacent servers towards the requested point  until the first of them reaches that point. %

The past decades have witnessed a rapid advancement of machine learning (ML) methods, which nowadays can be expected to predict often---but not always---uncertain data with good accuracy. The lack of guarantees on the predictions and %
the need for trustable performance guarantees lead to the area of {\em learning-augmented online algorithms}. This recently emerging research area investigates online algorithms that have access to predictions, e.g., on 
parts of the instance or the algorithm's execution, while not making any assumption on the quality of the predictions. Formally, we assume that a prediction has a certain quality~$\eta \geq 0$. In the context of learning theory one may think of the \emph{loss} of a prediction with respect to the ground truth. Accordingly,~$\eta = 0$ refers loosely speaking to the case where the prediction was correct. In the field of learning-augmented algorithm this quantity is called \emph{prediction error}. An algorithm does not know what quality a prediction has, but we can use it in the analysis to measure an algorithm's performance depending on~$\eta$.   If a learning-augmented algorithm is~$\mu(\eta)$-competitive for some function~$\mu$, we say that the algorithm is~$\alpha$-\emph{consistent} if~$\alpha =\mu(0)$ and~$\beta$-\emph{robust} if~$\mu(\eta) \leq \beta$ for any prediction with prediction error~$\eta$~\cite{purohit2018improving}. 

Very recently, Antoniadis et al.~\cite{antoniadis2020mts} proposed learning-augmented online algorithms for general metrical task systems, a generalization of our problem.  Their algorithm relies on simulating several online algorithms in parallel and keeping track of their solutions and cost. This technique crucially employs additional {\em memory} which 
can be a serious drawback in practice when decisions must be made without access to the history.

In this work, we introduce \defn{\memoryless}\ learning-augmented %
 algorithms for the~$k$-server problem on the line. An algorithm~$\A$ is intuitively \emph{\memoryless}, if the decision for the next move of~$\A$ only depends on the current situation {(server positions, request and prediction)}. It is especially independent of previous requests. However, as the algorithm is allowed to move a server to any point of the real line, it could use its position to encode any information at a negligible cost. This issue is often addressed by forbidding algorithms to move several servers per request~(hence, restricting to so-called \emph{lazy} algorithms) which leads to the classical \emph{memoryless} property, although variations of this definition exist~\cite{koutsCNN}. A downside of this restriction is that deterministic memoryless algorithms cannot be competitive, and there is no distinction between {the type of information gathered by \dc and unconstrained information encoding.} %
This difference has been nevertheless acknowledged by informally considering \dc as  memoryless%
~\cite{kouts09}, although noting immediately that such a definition for a non-lazy algorithm is cumbersome. In order to allow the behavior of \dc, we formally define \emph{\memoryless} algorithms as algorithms allowed to move several servers, making decisions independently of previous requests, but with an \emph{erasable} memory: for any set of~$k$ distinct points and any starting configuration, there exists a finite sequence of requests among these~$k$ points after which each point contains exactly one server. We will refer to such a sequence as a~\emph{force} {to these~$k$ points}. This definition is quite general as it allows to pre-move some servers as \dc does, and even allows information encoding, but provides a possibility to erase any information gathered. The algorithms we design will not abuse information encoding, but our lower bounds will hold in this context.

\paragraph*{Further related work}

The past few years have exhibited several
demonstrations of the power of learning-augmented algorithms improving on traditional online algorithms. %
Studied online problems include caching~\cite{lykouris2018caching,rohatgi2020caching,antoniadis2020mts,wei2020better}, paging~\cite{jiang2020paging}, ski rental~\cite{purohit2018improving,gollapudiP19,wangLW20,wei2020tradeoff, banerjee2020improving}, TCP acknowledgement~\cite{banerjee2020improving, bamas2020primal}, bin packing~\cite{angelopoulos2020online}, scheduling~\cite{purohit2018improving,bamas2020speedscaling,mitzenmacher2020scheduling,lattanzi2020scheduling,wei2020tradeoff,Im0QP21,AzarLT21flow}, secretary problems~\cite{antoniadis2020secretary,DuttingLLV21}, linear search~\cite{Angelopoulos21search}, matching~\cite{LavastidaMRX21instance-robust,KumarPSSV19semionline}, %
sorting~\cite{lu2021generalized}, online covering problems~%
\cite{bamas2020primal}, and possibly more by now.
Learning-augmented algorithms have proven to be successful also in other areas, e.g., to speed up search queries~\cite{kraska2018indexlearning}, in revenue optimization~\cite{munoz2017revenue}, to compute low rank approximations~\cite{indyk2019low-rank}, %
frequency estimation~\cite{hsu2019learning} and bloom filters~\cite{mitzenmacher2018bloom}.

More than a decade ago, Mahdian et al.~\cite{MahdianNS07} demonstrated performance improvements for online allocation algorithms when there is access to an accurate solution estimation. They further bounded the case where the estimation is inaccurate. While these bounds essentially correspond to consistency and robustness, they did not precisely measure the prediction quality. Yet they introduced a parameter to express %
the tradeoff between both bounds. In the recent field of learning-augmented algorithms, Kumar et al.~\cite{purohit2018improving} initiated the use of a similar 
parameter $\lambda \in [0,1]$.
It can be interpreted as an algorithm's indicator of trust in the given predictions: smaller $\lambda$ indicates stronger trust and gives a higher priority to a better consistency at the cost of a %
worse robustness, and vice versa. %
{Such parameterized consistency-robustness tradeoff has become standard for expressing the performance of learning-augmented algorithms when aiming for constant} \mbox{factors~\cite{purohit2018improving,wangLW20,wei2020tradeoff,angelopoulos2020online,banerjee2020improving,bamas2020primal,antoniadis2020secretary,Im0QP21}.}

As mentioned, Antoniadis~et~al.~\cite{antoniadis2020mts} provide %
a general learning-augmented framework for any metrical task systems which includes %
the~$k$-server problem. Applied to %
the line metric, they devise a learning-augmented algorithm that crucially requires %
memory and obtains a $9$-consistent and $9k$-robust algorithm. %

The~$k$-server problem has been studied also in the context of reinforcement learning (RL), originating at~\cite{junior2005k} and including hierarchical RL learning~\cite{costa16hierarchical} as well as deep RL learning~\cite{DBLP:journals/eswa/LinsNM19}.

The classical online $k$-server problem without access to predictions has been studied extensively, also in general metric spaces. The best known deterministic algorithm is the \textsc{WorkFunction} algorithm~\cite{koutsoupias1995workfunction} with a competitive ratio of~$2k-1$. For several special metric spaces there are even tighter bounds known for this algorithm~\cite{bartal04wf,zhang2020wfbound}. 
When allowing randomization, a~$\Omega(\log k / \log \log k)$ lower bound holds~\cite{BartalBM06} and a~${(\log k)}^{\mathcal{O}(1)}$-competitive randomized algorithm is conjectured~\cite{kouts09}.
Restricting further to memoryless randomized algorithms increases the lower bound on the competitive ratio exponentially to~$k$~\cite{kouts09} and some recent efforts focus on a more general variant in this setting~\cite{DBLP:journals/talg/ChiplunkarV20}.

The power of \dc goes beyond its optimality for the $k$-server problem in tree metrics~\cite{ChrobakL91}.  Recently, Buchbinder et al.~\cite{BuchbinderCN21} showed that it is a best possible deterministic algorithm for the more general~$k$-taxi problem, even in general metric spaces using an embedding into hierarchically separated~trees.

\paragraph*{Our contribution} 

We design learning-augmented \memoryless online algorithms for the $k$-server problem on the line. %
Firstly, we define some more notation and the precise prediction model.
We denote a server's name as well as its position on the line by~$s_i$, for $i\in \{1,2\ldots,k\}$. %
A configuration~$C_t = (s_1,\ldots,s_k) \in \mathbb{R}^k$ is a snapshot of the server positions at a certain point in time. 
For a given instance, a $k$-server algorithm outputs a sequence of configurations~$C_1,\ldots,C_n$ (also called \emph{schedule}) such that for every~$t = 1,\ldots,n$, we have~$r_t \in C_t$. We denote the initial configuration by~$C_0$. %
The objective function can be expressed as~$\sum_{t=1}^n d(C_{t-1}, C_t)$, where~$d(C_{t-1}, C_t)$ denotes the cost for moving the servers from~$C_{t-1}$ to~$C_t$.
We assume w.l.o.g.~$s_1 \leq \ldots \leq s_k$, as server overtakings can be uncrossed without increasing the~total~cost.

We employ a prediction model that predicts algorithmic choices of an optimal algorithm, that is predicting which server should serve a certain request. 
Given an instance~$I$ composed of the request sequence~$r_1,\ldots,r_n$, we define a~\emph{prediction} for~$I$ as a sequence of indices~$p_1,\ldots,p_n$ from the set~$\{1,\ldots,k\}$.
If~$s_1,\ldots,s_k$ are the servers of some learning-augmented algorithm, we call~$s_{p_t}$ the~\emph{predicted server} for the $t$-th request.
We call the algorithm that simply {\em follows the predictions} \ftp, that is, it serves each request by the predicted server (to simplify computations, we still remove overtakings as mentioned above, which is equivalent to relabel servers by their position order). We denote its cost by $\ftp(I)$. 
We define the \emph{prediction error}~$\eta = \ftp(I) - \opt(I)$ as quality measure for our predictions. Note that this error definition is independent of our algorithm.

Our main result is a parameterized algorithm for the $k$-server problem on the line with an error-dependent performance guarantee that---when having access to good-quality predictions---beats the known lower bound for deterministic online algorithms.

\begin{theorem}\label{thm:lambdaDC-competitive-ratio}
    Let~$\lambda \in [0,1]$. We define $\rob(k) = \sum_{i=0}^{k-1} \lambda^{-i}$, for $\lambda > 0 $, and $\rob(k)=\infty$, for~$\lambda=0$. Further, let 
    \[
        \cons(k) =
        \begin{cases}
            1 + 2\lambda + 2\lambda^2 + \ldots + 2\lambda^{(k-1)/2} 
             & \text{ if } k \text{ is odd} \\
            1 + 2\lambda + 2\lambda^2 + \ldots + 2\lambda^{k/2-1} + \lambda^{k/2}
            & \text{ if } k \text{ is even}.
        \end{cases}
    \]
    Let~$\eta$ denote the total prediction error and~$\opt$ the cost of an optimal solution.
    Then, there exists a learning-augmented \memoryless online algorithm for the~$k$-server problem on the line with a competitive ratio of at most
    \[
        \min \left\{ \cons(k)\left(1 + \frac{\eta}{\opt}\right), \rob(k) \right\}.
    \]
    In particular, the algorithm is~$\cons(k)$-consistent and~$\rob(k)$-robust, for~$\lambda >0$.
\end{theorem}

Interpreting both bounds as functions of~$\lambda \in [0,1]$ illustrates that~$\alpha(k)$ interpolates monotonously between~$1$ and~$k$ while~$\beta(k)$ grows from~$k$ as~$\lambda$ decreases. This matches our expectation on a learning-augmented online algorithm, as it improves in consistency but loses in robustness compared to the best possible online algorithm. From another
perspective, for a fixed value of~$\lambda$, $\alpha(k)$ is bounded by a constant (equal to $1+\frac 2 {1-\lambda})$ which highlights the algorithm consistency but this comes at the price of an exponential dependency on $k$ for $\beta(k)$.

To show this result, we design an algorithm that carefully balances between (i) the wish to simply follow the predictions (\ftp) which is obviously optimal if the predictions are correct,~i.e. is~$1$-consistent, and (ii) the best possible online algorithm when not having access to (good) predictions \dc~\cite{chrobak1991dc}, which is $k$-robust. An additional challenge is to preserve the \memoryless property. 
We achieve this, by generalizing the classical \dc~\cite{chrobak1991dc} in an intuitive way. Essentially, our algorithm \ldc includes the information about predicted servers and our trust into them by varying server~speeds.

The analysis of our algorithm is tight. %
On the technical side, our  analysis builds on the powerful {\em potential function method}, as does the analysis of the classical \dc~\cite{chrobak1991dc}. While %
\ldc is quite simple (a precise definition follows), the analysis is much more intricate and requires a careful re-design for the learning-augmented setting. Our main technical contribution is the definition and analysis of different parameterized potential functions for proving robustness and consistency, that capture the different speeds for moving servers and the {accordingly more difficult tracing of the server moves.}

We remark that our performance bound also holds (with an additional factor of $2$ on the error) using the error measure of Antoniadis et al.~\cite{antoniadis2020mts} for our problem~\cite{Lindermayr2020}. Their error definition sums up the distances between the configurations of $\opt$ and $\ftp$ after every request, thus, it may seem more intuitive as server positions are compared instead of solution costs.
However, our error definition allows to establish learnability results and also simplifies some analyses.

While our result is tailored to the $k$-server problem, the %
framework by Antoniadis et al.~\cite{antoniadis2020mts} is designed for more general metrical task systems. Interestingly, one of their methods is a
deterministic combination of \dc and \ftp, we refer to it as \ftpdc. It is shown %
that \ftpdc is $9$-consistent and~$9k$-robust. Our methods differ substantially. While \ftpdc carefully tracks states and costs of the simulated individual algorithms,  %
\ldc is a simple algorithm that %
only requires knowledge of the current configuration. %
Further, \ldc %
has a better performance for $k<20$ and an appropriate parameter $\lambda$ (e.g., $k=19$ and $\lambda=0.83$), 
but does not offer such a good tradeoff for larger~$k$. {Actually,} this is unavoidable for a certain class of %
\memoryless algorithms, that includes~\ldc.

Indeed, we complement our main result with an almost matching
lower bound on the consistency-robustness tradeoff. {We construct a
non-trivial bound for the class of \memoryless algorithms that satisfy an
additional locality property;} its precise definition is formulated in
\Cref{sec:tradeoff}. Intuitively, the locality property enforces an algorithm to
achieve a better competitive ratio for a subinstance served by fewer servers.
Other locality restrictions have been {required before to establish lower bounds}, 
e.g., for matching on the line, see~\cite{antoniadis18Matching}.

\begin{restatable}[]{theorem}{theoremTradeoffLB}\label{thm:pareto-k}
    Let~$\lambda \in (0,1]$,~$\rho(k) = \sum_{i=0}^{k-1} \lambda^{i}$ and~$\rob(k) = \sum_{i=0}^{k-1} \lambda^{-i}$.
    Let~$\A$ be a learning-augmented \local and \memoryless deterministic online algorithm for the~$k$-server problem on the line. Then, if~$\A$ is~$\rho(k)$-consistent, it is at least~$\rob(k)$-robust.
\end{restatable}

Algebraic transformations {(see \Cref{lemma:consistencies})} show that~$\alpha(k) < 2 \rho(k)$, %
which implies that %
{\ldc} achieves a tradeoff within a factor of at most~$2$ of the \emph{optimal} consistency-robustness tradeoff {(among \local and \memoryless algorithms)}. %
For~$k=2$, \ldc achieves the optimal~tradeoff {(among \memoryless algorithms)}. 

We demonstrate the %
power of our approach in empirical experiments %
on real-world data. We show that for a reasonable choice of~$\lambda$ our method outperforms the classical online algorithm \dc as well as the algorithm %
in~\cite{antoniadis2020mts} for nearly all prediction errors.

Finally, we address the learnability of our predictions, even though this is not the focus of our work. We show that a {\em static} prediction sequence is PAC-learnable~\cite{Valiant84,vapnik1971}. We show a bound on the sample complexity that is polynomial in the number of requests, $n$, and the number of servers, $k$, and we give a learning algorithm 
with a polynomial running time in $n, k$ and the number of samples.

\section{Algorithm and Roadmap for the Analysis}\label{sec:roadmap}

\paragraph{The Algorithm \ldc}

We generalize the classical \dc~\cite{chrobak1991dc}   %
by including the 
information about predicted servers as well as our trust into this advice, in an intuitive way. %
{If a request $r_t$ appears between two servers, the one closer to the predicted server $p_t$ moves by a greater distance towards the request---as if it traveled at a higher speed.  

Formally, we define \ldc for a given $\lambda\in [0,1]$ as follows.}
If~$r_t < s_1$ or~$r_t > s_k$, then \ldc only moves the closest server. 
Otherwise, we have~$s_{i} < r_t < s_{i+1}$. If~$p_t \leq i$, then \ldc moves~$s_i$ with speed 1 and~$s_{i+1}$ with speed~$\lambda$ towards~$r_t$ until one server reaches the request. %
If~$p_t \geq i+1$, the speeds of $s_i$ and $s_{i+1}$ are swapped. 
Hence, \ldc equals \ftp (with shortcuts) for~$\lambda = 0$, and \dc for~$\lambda = 1$. Using nonintegral values for~$\lambda$ gives an algorithm that interpolates between both.

\paragraph{Potential Function Analysis} The analysis of our algorithm  builds on the powerful {\em potential function method}, as does the analysis of the classical \dc~\cite{chrobak1991dc}.

Our potential analysis follows the well-known \emph{interleaving moves} technique~\cite{borodin98}. To compare two algorithms~$\A$ and~$\B$ in terms of competitiveness, we simulate both in parallel on some instance~$I$.
Then, we employ a potential function~$\Phi$ which maps at every time~$t$ the state of both algorithms (i.e. the algorithms current configurations) to a value~$\Phi_t \geq 0$, the potential at time~$t$. We define~$\Delta \Phi_t = \Phi_t - \Phi_{t-1}$. 
Let~$\Delta \B_t(I)$ resp.~$\Delta \A_t(I)$ denote the cost~$\A$ resp.~$\B$ charges for serving the request at time~$t$ and let~$\mu > 0$.
For every request~$r_t$, we assume that first~$\B$ serves the request, and second~$\A$. If
\begin{compactenum}[(i)]
    \item the move of~$\B$ increases~$\Phi$ by at most~$\mu \cdot \Delta \B_t(I)$, whereas
    \item the move of~$\A$ decreases~$\Phi$ by at least~$\Delta \A_t(I)$, 
\end{compactenum}
we can use a %
telescoping sum argument to conclude~$\A(I) \leq \mu \cdot \B(I) + \Phi_0$. Note that if~$\B$ is the optimal algorithm,~$\mu$ is equal to the competitive ratio of~$\A$ since~$\Phi_0$ only depends on $C_0$. %

To show an error-dependent competitive ratio in the learning-augmented setting, we follow three steps. We show first that the cost of \ldc is close to the cost of \ftp, that is~$\alg(I) \leq {\cons(k)} \cdot \ftp(I) + c$ for some~$c>0$ and for every instance~$I$. Note that this corresponds to the consistency case as \ftp is the optimal algorithm if~$\eta = 0$.
Second we plug in the definition of our prediction error~$\eta$ to bound the cost of~\ftp by the cost of the fixed optimal solution (fixed with respect to the definition of $\eta$) and~$\eta$.
Combining both results yields the first part of the competitive ratio of~\Cref{thm:lambdaDC-competitive-ratio}. Lastly we prove a robustness bound, i.e. a general bound independent of the prediction, on the cost of \ldc with respect to~$\opt$. %
All additive constants in the competitive ratios only depend on the initial configuration of the servers, being zero if all servers start at the same~position.

The potential functions we use to analyze \ldc are inspired by the potential function in the classical  analysis of \dc~\cite{chrobak1991dc}. %
{It is composed of a {\em matching part} $\Psi$, summing the distances between the server positions of an algorithm and the reference algorithm (\opt, \ftp) and a {\em spreadness part}~$\Theta$, summing the distances between an algorithms server positions.}
To incorporate the more sophisticated server moves at different speeds, we %
introduce multiplicative coefficients {to both parts}. The main technical contribution lies in identifying the proper weights and performing the much more involved~analysis.

\paragraph{Lower Bounds for \ldc} In \Cref{app:lDCrobust} we show that our analysis is tight.

\begin{restatable}{lemma}{lemmatightness}\label{lem:LBconsDC}
    \ldc is at least~$\cons(k)$-consistent and~$\rob(k)$-robust. 
\end{restatable}
\paragraph{Organization of the paper} For ease of exposition, 
we first consider the setting of~$2$ servers in~\Cref{sec:2-server}. Then, we extend the techniques to the general setting in~{\Cref{sec:k-server,sec:tradeoff}} %
while maintaining the same structure as for $k=2$. %
We illustrate and discuss the results of computational experiments in~\Cref{sec:experiments}, and, finally, talk about PAC learnability of our predictions in~\Cref{sec:learnability}.  

\section{Full Analysis for Two Servers}\label{sec:2-server}

\subsection{Error-dependent Competitive Ratio of \ldc} %

We show the theoretical guarantees of \ldc claimed in \Cref{thm:lambdaDC-competitive-ratio} restricted to two servers. We denote the cost of \ldc for some instance~$I$ by~$\alg(I)$, and the cost for serving a request~$r_t$ by~$\Delta \alg_t(I)$. If~$t$ is clear from the context then we omit the index.
\begin{theorem}\label{thm:lambdaDC-competitive-ratio-k2}
    For any parameter~$\lambda \in [0,1]$, \ldc has a competitive ratio of at most
    \[ 
        \min \left \{ (1 + \lambda)\left(1 + \frac{\eta}{\opt}\right), 1 + \frac{1}{\lambda} \right \}.
    \]
    Thus, it is~$(1 + \lambda)$-consistent and~$(1 + 1/\lambda)$-robust.
\end{theorem}

We follow the three-step approach outlined in the previous section. The definition of~$\eta$ immediately gives for any instance~$I$ and prediction with error~$\eta$ that~$\ftp(I) = \opt(I) + \eta$. %
With \Cref{lemma:k2-alg-to-prd,lem:k2-robustness} this implies \Cref{thm:lambdaDC-competitive-ratio-k2}.
We firstly compare the algorithm %
to \ftp.

\begin{lemma}\label{lemma:k2-alg-to-prd}
    For any instance~$I$ and~$\lambda \in [0,1]$, there is some $c \geq 0$ that only depends on the initial configuration such that~$\alg(I) \leq (1+\lambda) \cdot\ftp(I) + c$.
\end{lemma}
\begin{proof}
    Let~$I$ be an arbitrary instance {and let} servers start at positions $s_1^0$ and $s_2^0$. If~$\lambda = 0$, \ldc only shortcuts \ftp's moves, hence~$\alg(I) \leq \ftp(I)$. Now assume that~$\lambda > 0$. Let~$s_1,s_2$ be  \ldc's servers and~$x'_1,x'_2$ be \ftp's servers. We simulate~$I$ in parallel for both algorithms. At every time~$t$, we map the configurations of both algorithms to a non-negative value using the potential function
    \[
        \Phi~~ =~~ \underbrace{\frac{1 + \lambda}{\lambda} \left( \abs{s_1 - x'_1} + \abs{s_2 - x'_2} \right)}_{\Psi \text{ (matching part)}} ~~+ \underbrace{\abs{s_1 - s_2}}_{\Theta \text{ (spreadness part)}}.
    \]
    Suppose that %
    {a new} request arrives. First, \ftp serves the request. Assume that~$x'_1$ moves and charges cost~$\Delta \ftp$. Since \ldc remains in its previous configuration,~$\abs{x'_1 - s_1}$ increases by at most~$\Delta \ftp$, and~$\Phi$ increases by at most~$(1+\lambda)/\lambda \cdot \Delta \ftp$.
    Second, \ldc moves. Assume by scaling the instance that the algorithm serves the request after exactly one time unit,~i.e., the fast server moves distance~$1$ and the slow server distance~$\lambda$. We distinguish whether the request is between the algorithm's servers or not, and prove in each case that~$\Phi$ decreases by at least~$1 / \lambda \cdot \Delta \alg$.
    \begin{enumerate}[(a)]
        \item Suppose the request is not between the servers $s_1$ and $s_2$; say, it is left of $s_1$. Then \ldc moves only $s_1$ and $\Delta \alg =1$.
        Either~$x'_1$ or~$x'_2$ covers the request, hence moving~$s_1$ decreases~$\Psi$ by~$(1+\lambda)/\lambda$ while it increases~$\Theta$ by~$1$. Thus,
              \[
                  \Delta \Phi \leq - \frac{1 + \lambda}{\lambda} + 1 = -\frac{1}{\lambda} = -\frac{1}{\lambda} \cdot \Delta \alg.
              \]
        \item Suppose the request is between~$s_1$ and~$s_2$, and suppose that~$s_1$ is predicted. \ldc moves both servers and $\Delta \alg = 1+\lambda$. This means that~$x_1'$ already covers the request. Thus, moving~$s_1$ towards the request decreases~$\Psi$ by~$(1 + \lambda) / \lambda$, while~$s_2$ increases~$\Psi$ by at most~$(1 + \lambda) / \lambda \cdot \lambda$. Also,~$\Theta$ decreases by~$1+\lambda$. We can conclude that
              \[
                  \Delta \Phi \leq \frac{1 + \lambda}{\lambda} (-1 + \lambda) - (1 + \lambda) = -\frac{1}{\lambda}(1 + \lambda) = %
                  -\frac{1}{\lambda} \cdot \Delta \alg.
              \]
    \end{enumerate}
	Summing over all rounds, we obtain $\alg(I) \leq (1+\lambda)\ftp(I) + \lambda |s_1^0-s_2^0|$.
\end{proof}

Finally, we give a robustness guarantee for \ldc's performance independently of the prediction quality.
\begin{lemma}\label{lem:k2-robustness}
    For any instance~$I$ and~$\lambda \in (0,1]$, there is some~$c \geq 0$ that only depends on the initial configuration such that~$\alg(I) \leq (1 + 1/\lambda) \cdot \opt(I) + c$.
\end{lemma}
The proof of this claim is similar to the proof of~\Cref{lemma:k2-alg-to-prd} {with the crucial} difference that the reference algorithm is unknown. Hence, the multiplicative factor is larger but relative to the optimal solution and, thus, independent of the prediction error.
\begin{proof}
    Let~$I$ be an arbitrary instance and let~$\lambda \in (0,1]$.
    Let~$s_1, s_2$ be \ldc's servers and~$x_1, x_2$ the servers of an optimal algorithm. We define
    \[
        \Phi = \underbrace{(1 + \lambda) \left( \abs{s_1 - x_1} + \abs{s_2 - x_2} \right)}_{\Psi} + \underbrace{\abs{s_1 - s_2}}_{\Theta}.
    \]
    Upon arrival of a request, first the optimal algorithm moves and~$\Phi$ increases by at most $(1 + \lambda) \cdot \Delta \opt$. Second \ldc moves and, by scaling the instance, we assume that the request is served after exactly one time unit. We distinguish whether the request is between the algorithm's servers or not, and show that in each case~$\Phi$ decreases by at least~$\lambda \cdot \Delta \alg$.
    \begin{enumerate}[(a)]
        \item Let the request be not between the servers, say on the left of~$s_1$. Either~$x_1$ or~$x_2$ covers the request, hence moving~$s_1$ decreases~$\Psi$ by~$1+\lambda$ while it increases~$\Theta$ by~$1$. Thus,
              \[
                  \Delta \Phi \leq - (1 + \lambda) + 1 = -\lambda = -\lambda \cdot \Delta \alg.
              \]
        \item Let the request be between~$s_1$ and~$s_2$, and suppose that~$s_1$ is predicted. The request is covered by~$x_1$ or~$x_2$. In the worst case ($x_2$ covers the request), moving~$s_1$ towards the request increases~$\Psi$ by at most~$1+\lambda$, while~$s_2$ decreases~$\Psi$ only by~$(1+\lambda) \lambda$. Also,~$\Theta$ decreases by~$1+\lambda$. Put together,
              \[
                  \Delta \Phi \leq (1+\lambda) (1 - \lambda) - (1 + \lambda) = -\lambda (1 + \lambda) = -\lambda \cdot \Delta \alg.\qedhere
              \]
    \end{enumerate}
\end{proof}

\subsection{Optimality of \ldc: the Consistency-Robustness Tradeoff}\label{sec:k2-tradeoff}

We now show that \ldc is optimal for two servers, in the sense that no \memoryless algorithm can achieve a better robustness-consistency tradeoff. As we target \memoryless algorithms, {at any time, we can use \emph{force} requests, {cf., \cref{sec:intro},} to enforce the algorithm to place its servers at prescribed locations.} %

\begin{theorem}\label{thm:pareto-k2}
    Let~$\A$ be a learning-augmented \memoryless algorithm for the~$2$-server problem on the line and let~$\lambda \in (0,1]$. If~$\A$ is~$(1+\lambda)$-consistent, it is at least~$(1 + 1/\lambda)$-robust.
\end{theorem}

\begin{proof}%
    Let~$\lambda \in (0,1]$ and~$\A$ be a~$(1+\lambda)$-consistent, \memoryless algorithm for the~$2$-server problem on the line. This means for every instance~$I$,~$\A(I) \leq (1+\lambda) \cdot \opt(I) + \nu$ if~$\eta = 0$, where $\nu$ depends on the initial configuration. Let~$a, b$ and~$c$ be consecutive points on the line at position $-1$, $0$ and $L\geq 1+1/\lambda$, and~$(a,b)$ the algorithm's initial configuration. 
 
    Consider the instance~$I^{\infty}$ which is composed of a force to $(a,c)$, followed by arbitrarily many alternating requests at~$b$ and~$a$. Clearly, an optimal solution for instance~$I^{\infty}$ is to move the right server to~$c$ and then immediately back to~$b$ with a total cost of~$2L$.
    
    Assume that~$\A$ gets this optimal solution as prediction.~$\A$ moves one server to~$c$ for the first request. Since the consistency implies that~$\A(I^{\infty}) \leq (1+\lambda) \opt$, at some point in time~$\A$ has to move the right server to $b$. Denote the instance which ends at this point in time by~$I$. Note that~$\A(I^{\infty}) \geq \A(I)$. Let~$n_L$ denote the number of times in instance~$I$ where the left server moves from~$a$ to~$b$ and back to~$a$ (cost of 2). Since the right server pays at least~$L$ for moving from~$c$ to~$b$, we conclude~$\A(I) \geq 2 n_L + 2 L$. The consistency of $\A$ leads to $2 n_L + {2}L \leq (1 + \lambda)2 L+\nu$, which means $n_L\leq \lambda L + \nu / 2$. %

    We now construct another instance~$I^\omega$ by concatenating~$\omega$ copies of instance~$I$, each starting by the force to $(a,c)$. We call such a copy an~\emph{iteration}, and in each iteration we use the same predictions as in instance~$I$. $\A$ has to pay at least $L$ for the force, as the right server was previously on $b$, and then $\A$ follows the same behavior as in $I$ in each iteration. So $ \A(I^\omega) \geq \omega \cdot (2n_L + 2 L)$. 
    Another solution for instance~$I^\omega$ is to move the right server to~$c$ in the beginning with cost~$L$  and leave it there, while the left server alternates between~$a$ and~$b$. Hence,~$\opt(I^\omega) \leq L + \omega \cdot 2(n_L+1)$. Indeed, $b$ is requested $n_L+1$ times per iteration: $n_L$ where $\A$ uses the left server and one where it uses the right server. The ratio is then   
   \[
          \frac{\A(I^\omega)}{\opt(I^\omega)} \geq \frac{\omega \cdot (2n_L + 2 L)}{L + \omega \cdot 2(n_L+1)} \xrightarrow{\omega \to \infty} \frac{2n_L + 2 L}{2(n_L+1)} = 1 + \frac{L-1}{n_L+1} \geq 1+\frac{L-1}{\lambda L + \frac{\nu}{2} + 1} \xrightarrow{L \to \infty} 1+\frac 1\lambda\,,
    \]
    which implies that~$\A$ is at least~$(1+ 1 / \lambda)$-robust.
\end{proof}

\section{The General Case with $k$ Servers: Upper Bound}\label{sec:k-server}

We present two lemmas which imply%
~\Cref{thm:lambdaDC-competitive-ratio}. The novelty lies in designing appropriate potential functions that capture the server movements at different speeds. %
This takes substantially more technical care than in the $2$-server case but builds on the same ideas.

In the first step of the analysis, we compare the performance of \ldc and~\ftp.

\begin{restatable}{lemma}{LemGenConsistency}\label{lemma:general-consistency-upper-bound}
    For every instance~$I$ and~$\lambda \in [0,1]$, there is some~$c > 0$ that only depends on the initial configuration such that~$\alg(I) \leq \cons(k) \cdot \ftp(I) + c$.
\end{restatable}
Let~$I$ be an arbitrary instance. Note that~$\lambda = 0$ implies~$\alg(I) \leq \ftp(I)$ as \ldc can only shortcut \ftp's moves. So, we now assume that~$\lambda \in (0,1]$.
We define a new potential function~$\Phi$ as follows. 
Let~$s_1, \ldots, s_k$ be the servers of \ldc and let~$x_1', \ldots, x_k'$ be the servers of \ftp. 
For~$1 \leq i < j \leq k$ and~$\ell = \min\{j-i, k-(j-i)\}-1$ we define~$\delta_{ij} = \lambda^{\ell }$, see \Cref{fig:weights}.~Then,
\[
    \Phi =  \underbrace{\frac{\cons(k)}{\lambda} \cdot\sum_{i=1}^k \abs{s_i - x_i'}}_\Psi  + \underbrace{\sum_{i < j} \delta_{ij} \abs{s_i - s_j}.}_\Theta
\]
\begin{figure}[tb]
	\centering
	\begin{adjustbox}{width=0.65\linewidth}
		\begin{tikzpicture}[scale=0.75]
			\node[circle, draw, thick, minimum size=8mm, inner sep=1pt, outer sep=0pt] (s1) at (0,0) {$s_1$};
			\node[circle, draw, thick, minimum size=8mm, inner sep=1pt, outer sep=0pt] (s2) at (2.5,0) {$s_2$};
			\node[circle, draw, thick, minimum size=8mm, inner sep=1pt, outer sep=0pt] (s3) at (5,0) {$s_3$};
			\node[circle, draw, thick, minimum size=8mm, inner sep=1pt, outer sep=0pt] (s4) at (7.5,0) {$s_4$};
			\node at (10,0) {$\cdots$};
			\node[circle, draw, thick, minimum size=8mm, inner sep=1pt, outer sep=0pt] (sk-2) at (12.5,0) {$s_{k-2}$};
			\node[circle, draw, thick, minimum size=8mm, inner sep=1pt, outer sep=0pt] (sk-1) at (15,0) {$s_{k-1}$};
			\node[circle, draw, thick, minimum size=8mm, inner sep=1pt, outer sep=0pt] (sk) at (17.5,0) {$s_k$};
		
			\newcommand{\bendparam}{40}
			
			\draw[line width=3.3pt] (s1) edge node[above, pos=0.5, xshift=3pt] {$\ell = 0$} (s2);
			\draw[line width=2.6pt, algBlue] (s1) edge[bend left=\bendparam] node[above, pos=0.9, xshift=8pt] {$\ell = 1$} (s3);
			\draw[line width=1.9pt, algOrange] (s1) edge[bend left=\bendparam] node[above, pos=0.9, xshift=8pt] {$\ell = 2$} (s4);
			\draw[line width=1.9pt, algOrange] (s1) edge[bend left=\bendparam] node[above, pos=0.95, xshift=8pt] {$\ell = 2$} (sk-2);
			\draw[line width=2.6pt, algBlue] (s1) edge[bend left=\bendparam] node[above, pos=0.96, xshift=9pt] {$\ell = 1$} (sk-1);
			\draw[line width=3.3pt] (s1) edge[bend left=\bendparam] node[above, pos=0.97, xshift=10pt] {$\ell = 0$} (sk);
		
			\draw[line width=3.3pt] (s2) edge node[below, pos=0.5, xshift=3pt] {$\ell = 0$} (s3);
			\draw[line width=2.6pt, algBlue] (s2) edge[bend right=\bendparam] node[below, pos=0.9, xshift=8pt] {$\ell = 1$} (s4);
			\draw[line width=1.2pt, algGreen] (s2) edge[bend right=\bendparam] node[below, pos=0.95, xshift=8pt] {$\ell = 3$} (sk-2);
			\draw[line width=1.9pt, algOrange] (s2) edge[bend right=\bendparam] node[below, pos=0.96, xshift=8pt] {$\ell = 2$} (sk-1);
			\draw[line width=2.6pt, algBlue] (s2) edge[bend right=\bendparam] node[below, pos=0.97, xshift=8pt] {$\ell = 1$} (sk);
		\end{tikzpicture}
	\end{adjustbox}
	\caption{Visualization of all incident~$\delta_{ij}$-weights of the servers~$s_1$ and~$s_2$. The thickness (resp.\ color) of an arc indicates the influence of the corresponding distance in~$\Phi$.}\label{fig:weights}
\end{figure}
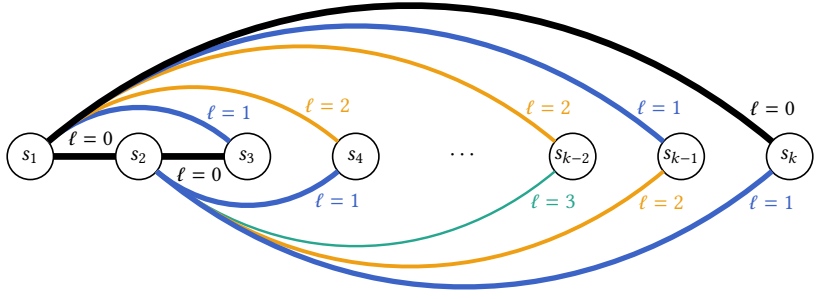  %
{ Intuitively,} the leading coefficient of $\Psi$ comes from the targeted competitive ratio. Then, in $\Theta$, the coefficient in front of each term depends on the number of interleaving servers. Following the idea of \Cref{lem:LBconsDC}, when \ldc moves a server by a distance of 1 as in \opt, its neighbor moves by a distance of $\lambda$. Hence, correcting the position of this neighbor means that the next server moves by a distance $\lambda^2$. Therefore, this geometric decrease in the consequences of a movement also appears in the expression of $\Theta$. The symmetric increase when $j-i$ grows is more difficult to explain intuitively, but is required to compensate the modifications of $\Psi$. {The coefficients of $\Theta$ are illustrated in \Cref{fig:weights}.}

We carefully analyze in~{\Cref{app:upper-bounds}}
how the potential %
changes when \ftp\ and \ldc move servers.
Further, we give a robustness guarantee for \ldc for any %
error.
\begin{restatable}{lemma}{LemGenRobustness}\label{lemma:rob}
	For any instance~$I$ and~$\lambda \in (0,1]$, there is some~$c \geq 0$ that only depends on the initial configuration such that~$\alg(I) \leq \rob(k) \cdot \opt(I) + c$.
\end{restatable}
Proving the general upper bound on the competitive ratio, independent of the prediction error, is much more intricate than in the two-server case and than the consistency proof. Again, our key ingredient is a carefully chosen  potential function~$\Phi$. We generalize the function used for the consistency bound even further by refining the weights, in particular, adding server-dependent weights to the term $\Psi$ measuring the distance between the positions of the algorithm's servers and the optimal servers.

Let~$\lambda \in (0,1]$. Fix~$k$, let~$\beta = \beta(k)=\sum_{i=0}^{k-1}\lambda^{-i}$, and let~$s_1, \ldots, s_k$ be the servers of \ldc and let~$x_1, \ldots, x_k$ be the servers of an optimal solution. The potential function~is
\[
	\Phi = \underbrace{\rob \gamma \left(\sum_{i=1}^k \omega_i \abs{s_i - x_i} \right)}_{\Psi} + \underbrace{\sum_{i<j} \delta_{ij} \abs{s_i - s_j}}_{\Theta}.
\]
We specify the weights in this function as follows. For a pair of servers~$s_i,s_j$ with~$1 \leq i < j \leq k$, let~$\ell = \min\{j-i, k-(j-i)\}-1$ and
$\displaystyle 
	\delta_{ij} %
= (\lambda^\ell + \lambda^{k - 2 -\ell})/(1 + \lambda^{k-2}).
$

The intuition of the weights in the spreadness part $\Theta$ is the same as %
in the consistency potential function above. However, the new %
weights~$\omega_i$ in the matching part $\Psi$ (defined below) require %
the more complex weights $\delta_{ij}$ compared to the simpler~$\lambda^{\ell}$ weights. %

{Further, we define $d_{\ceil{k/2}} = 0$ if $k$ is odd and for all~$1 \leq i \leq \floor{k/2}$ let}
\[
	d_i = d_{k+1-i} = \frac{2}{1 + \lambda^{k-2}} \sum_{\ell=i-1}^{k-1-i} \lambda^\ell.
\]
{We demonstrate in the appendix %
that these values correspond to the change of~$\Theta$ when a server of \ldc moves.}
Let~$\gamma = d_1 / (\beta - 1)$,~$\omega_1 = \omega_k = 1$ and for~$2 \leq i \leq \ceil{k/2}$ we define the server-individual weights
	\begin{align*}
		\omega_i = {\omega_{k+1-i}} = \begin{dcases}
			\frac{2 \lambda \sum_{j=1}^{i/2-1} d_{2j} - 2 \sum_{j=1}^{i/2-1} d_{2j+1} + \lambda d_i + (2 + \lambda) \gamma}{\rob \gamma \lambda} & \text{ if } i \text{ is even, and} \\
			\frac{2 \lambda \sum_{j=1}^{(i-1)/2} d_{2j} - 2 \sum_{j=1}^{(i-3)/2} d_{2j+1} - d_i + \gamma}{\rob \gamma}                           & \text{ if } i \text{ is odd.}      \\
		\end{dcases}
	\end{align*}

We finally prove~\Cref{lemma:rob} in~\cref{app:upper-bounds} by exhaustively reviewing all possible moves and bounding the corresponding change of~$\Phi$. Establishing a constant upper bound of the~$\omega$-weights yields a general upper bound on the increase of~$\Phi$ independently of the choice of the optimal solution's server. 
We further choose the scaling parameter~$\gamma$ such that the decrease of~$\Phi$ exactly matches the required lower bound for the case where the request is outside of the convex hull of \ldc's servers.
The remaining cases are split among the possible locations where a request can appear between two servers of \ldc, and we show in each case that~$\Phi$ decreases enough. Intuitively, {the~$\omega$ values are defined} such that a wrong prediction gives a tight bound on the decrease of~$\Phi$ for \ldc's move, while a correct prediction still guarantees a loose bound.

\section{The Consistency-Robustness Tradeoff}\label{sec:tradeoff} 

In this section we give a bound on the consistency-robustness tradeoff, as stated in \Cref{thm:pareto-k}. Our bound holds for \memoryless algorithms that satisfy a certain locality property, which includes \ldc. Informally, we require that a $k$-server algorithm with a certain consistency~$\mu(k)$ shall have a consistency~$\mu(k')$ on a sub-instance that it serves with $k'<k$ servers. The rationale is to prevent the mere presence of additional unused workers to allow the algorithm to perform poorly on a subinstance served by few servers, as $\mu(k')<\mu(k)$. Hence, such algorithms are expected to present a better performance on a modified instance where some extreme servers are removed and side-effects due to their presence are simulated.
In the following, we make this intuition precise and sketch our worst-case~construction.

Given an algorithm $\A$ which is $\mu(k)$-consistent for the $k$-server problem, we define the notion of \defn{\local}. Given an instance of the $k$-server problem served by algorithm $\A$, consider any subset~$S'$ of $k'$ consecutive servers. We construct an instance $I'$ of the $k'$-server problem based on $I$ and $S'$:  If a request of $I$ is predicted
to be served by a server in $S'$ then this request is replicated in $I'$.
Otherwise, $I'$ requests the position of the closest server among $S'$ after
$\A$ served this request in $I$ (in order to take into account side-effects due to additional servers in the original instance). Let~$\ftp(I')$ be the cost of solving~$I'$
following the original predictions of $I'$ (using the closest server among $S'$ if a server outside of $S'$ was initially predicted). An algorithm is \emph{\local} if its total cost on $I$ restricted to the servers in $S'$ is at most $\mu(k')\cdot \ftp(I') + c$, where $c$ {can be upper bounded} based only on the initial configuration. We further require that if the initial and final configurations differ by a total distance of $\varepsilon$, then $c=O(k'\varepsilon)$.
Note that $\ldc$ is \local as its behavior in $I$ restricted to the servers in
$S'$ is equal to its behavior in~$I'$ with $k'$ servers.

The proof of \Cref{thm:pareto-k} generalizes ideas from the $2$-server case (\Cref{sec:k2-tradeoff}) in a highly non-trivial way. We only sketch the main idea and refer to \Cref{app:tradeoff} for details.
Let \A be a \memoryless and \local deterministic algorithm. We construct an instance that %
starts with $k$ equidistant servers. First, a point far on the right is requested. Then the initial server
locations are requested following specific rules until the rightmost server
comes back. Predictions correspond to the server initially at the point
requested. The consistency of \A limits the possible cost paid before 
the rightmost server comes back.
The \local definition allows, with technical care, to link the distance
traveled by two neighboring servers: the left one travels a total distance at most
$\lambda$ times the right one {(plus negligible terms)}. An offline solution can afford to initially shift
all servers to the right, and then move only the leftmost server, which \A could
not move much. We then repeat this instance, and use the
\memoryless and deterministic characteristics of~\A to eliminate constant costs
and show the desired robustness lower bound, again with technical care.

\section{Experiments}\label{sec:experiments}
\begin{figure}[tb]
    \begin{subfigure}[t]{0.47\textwidth}
        \includegraphics[width=\textwidth]{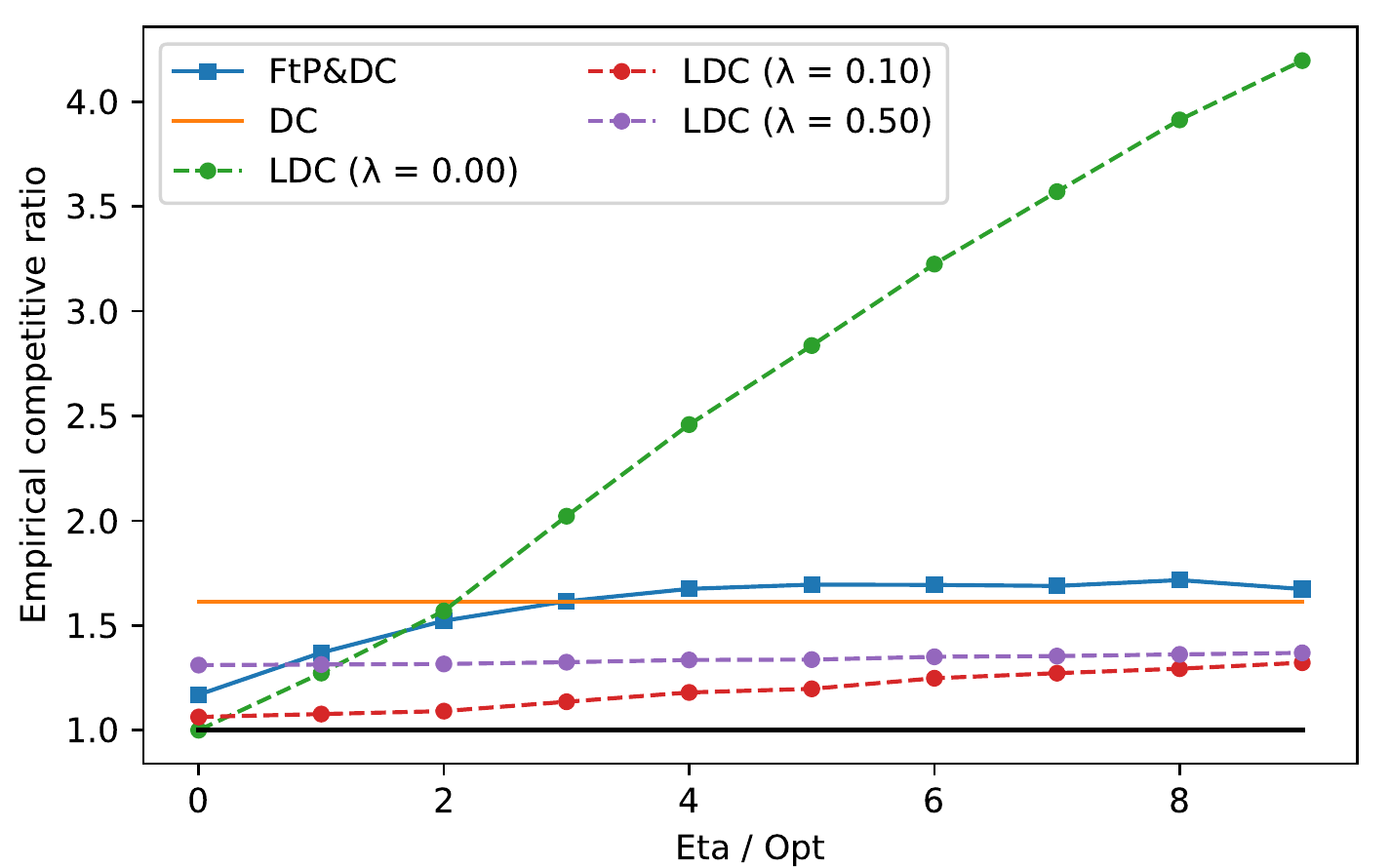}
        \caption{Results for $k=2$.}
    \end{subfigure}\hfill
    \begin{subfigure}[t]{0.47\textwidth}
        \includegraphics[width=\textwidth]{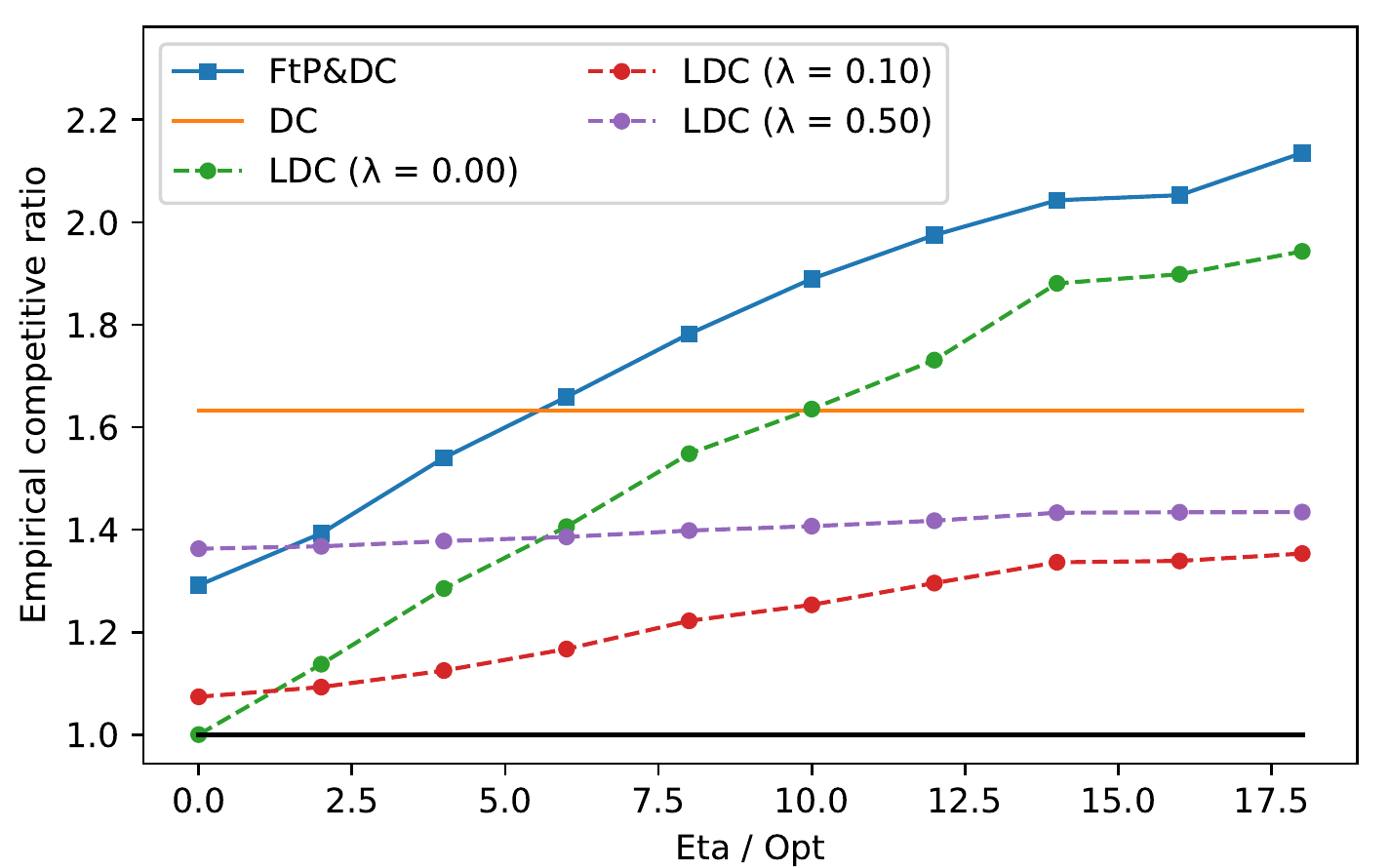}
        \caption{Results for $k=10$.}
    \end{subfigure}
    \caption{Non-lazy algorithms: means over all empirical competitive ratios per prediction quality.}%
    \label{fig:experimental-results}
\end{figure}

We supplement our theoretical results by empirically comparing our learning-augmented algorithm \ldc with the classical online algorithm ignoring predictions \dc~\cite{chrobak1991dc} and the previously proposed prediction-based algorithm 
	 \ftpdc~\cite{antoniadis2020mts} on real world data. %
We generate instances with~$1000$ requests 
based on the BrightKite-Dataset~\cite{cho2011friendship}, which is composed of sequences of %
coordinates of app check-ins. %
This dataset was used previously %
to evaluate and compare learning-augmented algorithms for caching problems~\cite{antoniadis2020mts,lykouris2018caching}.
We further generate predictions in a semi-random fashion aiming for large and evenly distributed prediction errors.
All algorithms are implemented in lazy and non-lazy variants.

The results for non-lazy implementations are displayed in~\Cref{fig:experimental-results}.
They show well that, for a reasonable choice of~$\lambda$ ($0.1 \leq \lambda \leq 0.5$), \ldc outperforms both \dc and \ftpdc for almost all generated relative prediction errors. This is true even if laziness is allowed as we show in~\Cref{app:experiments}.
We give also more details on the generation of instances and predictions, as well as an overview over all results.

\section{PAC Learnability of Predictions}\label{sec:learnability}
While our results show the applicability of untrusted predictions, it is a natural question whether such predictions are actually learnable. 

In \Cref{app:learnability}, we show that for our model a {\em static} prediction sequence is PAC learnable in an agnostic sense using empirical risk minimization. 
That is, given an unknown distribution over request sequences which we can sample, we can find a prediction that is close to the best possible prediction for this distribution in terms of prediction error using a bounded number of samples.

\begin{restatable}{theorem}{ThmLearnability}\label{thm:learnability}
			For any~$\epsilon, \delta \in (0,1)$, a known initial configuration~$C_0$ and any distribution~$\D$ over the sequences of~$n$ requests of known extent, there exists an algorithm which, given an i.i.d. sample of~$\D$ of size~$m \in \bigO \left(\frac{1}{\epsilon^2}\cdot {(n \log k - \log \delta)\eta_{\max}^2}\right)$, returns a prediction~$\pred_p \in \Hyp$ in polynomial time depending on~$k$,~$n$ and~$m$, such that with probability of at least~$(1-\delta)$ it holds~$\E_{\seq \sim \D}[\eta_\seq(\pred_p)] \leq \E_{\seq \sim \D}[\eta_{\seq}(\pred^*)] + \epsilon$, where~$\pred^* = \arg \min_{\pred \in \Hyp} \E_{\seq \sim \D}[\eta_\seq(\pred)]$.
\end{restatable}

We remark that a pre-computed static prediction does not include information about the partially revealed input.
Thus, this is a rather weak prediction and may not help \ldc much. 
The existence of an {\em adaptive} prediction policy which can be efficiently learned remains an open question. Such a policy would provide much more valuable information to our learning-augmented online algorithm.
\section{Conclusion} We show the power of (untrusted) predictions in designing online algorithms for the $k$-server problem on the line. Our algorithm generalizes the classical \dc algorithm~\cite{chrobak1991dc} in an intuitive way and admits a (nearly) tight error-dependent competitive analysis, based on new potential functions, and outperforms other methods from the literature. While we can show PAC learnability for static predictions, we leave open whether possibly more powerful adaptive prediction models are learnable. %

Clearly, it would be interesting to see whether our results generalize to more general metric spaces than the line. In fact, in a related version we show that our upper bounds for the $2$-server problem can be extended to tree metrics~\cite{Lindermayr2020} %
and we expect that an extension to $k$ servers is possible. %
However, for more general metrics our current approach seems not to generalize well. Further, we focused on \memoryless algorithms, leaving open a more precise quantification of the power of memory. Finally, the recent success on  randomized k-server algorithms~\cite{BubeckCLLM18} raises the question whether and how randomized algorithms can benefit from (ML) predictions.

\newpage
\bibliography{literature}

\begin{thebibliography}{10}

\bibitem{Angelopoulos21search}
Spyros Angelopoulos.
\newblock Online search with a hint.
\newblock In {\em {ITCS}}, volume 185 of {\em LIPIcs}, pages 51:1--51:16.
  Schloss Dagstuhl - Leibniz-Zentrum f{\"{u}}r Informatik, 2021.

\bibitem{angelopoulos2020online}
Spyros Angelopoulos, Christoph D{\"{u}}rr, Shendan Jin, Shahin Kamali, and
  Marc~P. Renault.
\newblock Online computation with untrusted advice.
\newblock In {\em {ITCS}}, volume 151 of {\em LIPIcs}, pages 52:1--52:15.
  Schloss Dagstuhl - Leibniz-Zentrum f{\"{u}}r Informatik, 2020.

\bibitem{antoniadis2020mts}
Antonios Antoniadis, Christian Coester, Marek Eli{\'{a}}s, Adam Polak, and
  Bertrand Simon.
\newblock Online metric algorithms with untrusted predictions.
\newblock In {\em {ICML}}, volume 119 of {\em Proceedings of Machine Learning
  Research}, pages 345--355. {PMLR}, 2020.

\bibitem{antoniadis18Matching}
Antonios Antoniadis, Carsten Fischer, and Andreas T{\"o}nnis.
\newblock A collection of lower bounds for online matching on the line.
\newblock In {\em Latin American Symposium on Theoretical Informatics}, pages
  52--65. Springer, 2018.

\bibitem{antoniadis2020secretary}
Antonios Antoniadis, Themis Gouleakis, Pieter Kleer, and Pavel Kolev.
\newblock Secretary and online matching problems with machine learned advice.
\newblock In {\em {NeurIPS}}, 2020.

\bibitem{AzarLT21flow}
Yossi Azar, Stefano Leonardi, and Noam Touitou.
\newblock Flow time scheduling with uncertain processing time.
\newblock In {\em {STOC}}, pages 1070--1080. {ACM}, 2021.

\bibitem{bamas2020speedscaling}
{\'{E}}tienne Bamas, Andreas Maggiori, Lars Rohwedder, and Ola Svensson.
\newblock Learning augmented energy minimization via speed scaling.
\newblock In {\em {NeurIPS}}, 2020.

\bibitem{bamas2020primal}
{\'{E}}tienne Bamas, Andreas Maggiori, and Ola Svensson.
\newblock The primal-dual method for learning augmented algorithms.
\newblock In {\em {NeurIPS}}, 2020.

\bibitem{banerjee2020improving}
Soumya Banerjee.
\newblock Improving online rent-or-buy algorithms with sequential decision
  making and {ML} predictions.
\newblock In {\em {NeurIPS}}, 2020.

\bibitem{BartalBM06}
Yair Bartal, B{\'{e}}la Bollob{\'{a}}s, and Manor Mendel.
\newblock Ramsey-type theorems for metric spaces with applications to online
  problems.
\newblock {\em J. Comput. Syst. Sci.}, 72(5):890--921, 2006.

\bibitem{bartal04wf}
Yair Bartal and Elias Koutsoupias.
\newblock On the competitive ratio of the work function algorithm for the
  k-server problem.
\newblock {\em Theor. Comput. Sci.}, 324(2-3):337--345, 2004.

\bibitem{borodin98}
Allan Borodin and Ran El{-}Yaniv.
\newblock {\em Online computation and competitive analysis}.
\newblock Cambridge University Press, 1998.

\bibitem{BubeckCLLM18}
S{\'{e}}bastien Bubeck, Michael~B. Cohen, Yin~Tat Lee, James~R. Lee, and
  Aleksander Madry.
\newblock k-server via multiscale entropic regularization.
\newblock In {\em {STOC}}, pages 3--16. {ACM}, 2018.

\bibitem{BuchbinderCN21}
Niv Buchbinder, Christian Coester, and Joseph~(Seffi) Naor.
\newblock Online k-taxi via double coverage and time-reverse primal-dual.
\newblock In {\em {IPCO}}, volume 12707 of {\em Lecture Notes in Computer
  Science}, pages 15--29. Springer, 2021.

\bibitem{DBLP:journals/talg/ChiplunkarV20}
Ashish Chiplunkar and Sundar Vishwanathan.
\newblock Randomized memoryless algorithms for the weighted and the generalized
  \emph{k}-server problems.
\newblock {\em {ACM} Trans. Algorithms}, 16(1):14:1--14:28, 2020.

\bibitem{cho2011friendship}
Eunjoon Cho, Seth~A. Myers, and Jure Leskovec.
\newblock Friendship and mobility: user movement in location-based social
  networks.
\newblock In {\em {KDD}}, pages 1082--1090. {ACM}, 2011.

\bibitem{chrobak1991dc}
Marek Chrobak, Howard~J. Karloff, T.~H. Payne, and Sundar Vishwanathan.
\newblock New results on server problems.
\newblock {\em {SIAM} J. Discret. Math.}, 4(2):172--181, 1991.

\bibitem{ChrobakL91}
Marek Chrobak and Lawrence~L. Larmore.
\newblock An optimal on-line algorithm for k-servers on trees.
\newblock {\em {SIAM} J. Comput.}, 20(1):144--148, 1991.

\bibitem{DuttingLLV21}
Paul D{\"{u}}tting, Silvio Lattanzi, Renato~Paes Leme, and Sergei
  Vassilvitskii.
\newblock Secretaries with advice.
\newblock In {\em {EC}}, pages 409--429. {ACM}, 2021.

\bibitem{gollapudiP19}
Sreenivas Gollapudi and Debmalya Panigrahi.
\newblock Online algorithms for rent-or-buy with expert advice.
\newblock In {\em {ICML}}, volume~97 of {\em Proceedings of Machine Learning
  Research}, pages 2319--2327. {PMLR}, 2019.

\bibitem{hsu2019learning}
Chen{-}Yu Hsu, Piotr Indyk, Dina Katabi, and Ali Vakilian.
\newblock Learning-based frequency estimation algorithms.
\newblock In {\em {ICLR}}, 2019.

\bibitem{Im0QP21}
Sungjin Im, Ravi Kumar, Mahshid~Montazer Qaem, and Manish Purohit.
\newblock Non-clairvoyant scheduling with predictions.
\newblock In {\em {SPAA}}, pages 285--294. {ACM}, 2021.

\bibitem{indyk2019low-rank}
Piotr Indyk, Ali Vakilian, and Yang Yuan.
\newblock Learning-based low-rank approximations.
\newblock In {\em {NeurIPS}}, pages 7400--7410, 2019.

\bibitem{jiang2020paging}
Zhihao Jiang, Debmalya Panigrahi, and Kevin Sun.
\newblock Online algorithms for weighted paging with predictions.
\newblock In {\em {ICALP}}, volume 168 of {\em LIPIcs}, pages 69:1--69:18.
  Schloss Dagstuhl - Leibniz-Zentrum f{\"{u}}r Informatik, 2020.

\bibitem{junior2005k}
Manoel Leandro~L Junior, AD~Doria Neto, and Jorge~D Melo.
\newblock The k-server problem: a reinforcement learning approach.
\newblock In {\em {IJCNN}}, 2005.

\bibitem{kouts09}
Elias Koutsoupias.
\newblock The k-server problem.
\newblock {\em Comput. Sci. Rev.}, 3(2):105–118, May 2009.

\bibitem{koutsoupias1995workfunction}
Elias Koutsoupias and Christos~H. Papadimitriou.
\newblock On the k-server conjecture.
\newblock {\em J. {ACM}}, 42(5):971--983, 1995.

\bibitem{koutsCNN}
Elias Koutsoupias and David~Scot Taylor.
\newblock The {CNN} problem and other k-server variants.
\newblock {\em Theoretical Computer Science}, 324(2):347 -- 359, 2004.
\newblock Online Algorithms: In Memoriam, Steve Seiden.

\bibitem{kraska2018indexlearning}
Tim Kraska, Alex Beutel, Ed~H. Chi, Jeffrey Dean, and Neoklis Polyzotis.
\newblock The case for learned index structures.
\newblock In {\em {SIGMOD} Conference}, pages 489--504. {ACM}, 2018.

\bibitem{KumarPSSV19semionline}
Ravi Kumar, Manish Purohit, Aaron Schild, Zoya Svitkina, and Erik Vee.
\newblock Semi-online bipartite matching.
\newblock In {\em {ITCS}}, volume 124 of {\em LIPIcs}, pages 50:1--50:20.
  Schloss Dagstuhl - Leibniz-Zentrum f{\"{u}}r Informatik, 2019.

\bibitem{lattanzi2020scheduling}
Silvio Lattanzi, Thomas Lavastida, Benjamin Moseley, and Sergei Vassilvitskii.
\newblock Online scheduling via learned weights.
\newblock In {\em {SODA}}, pages 1859--1877. {SIAM}, 2020.

\bibitem{LavastidaMRX21instance-robust}
Thomas Lavastida, Benjamin Moseley, R.~Ravi, and Chenyang Xu.
\newblock Learnable and instance-robust predictions for online matching, flows
  and load balancing.
\newblock {\em To appear at ESA}, abs/2011.11743, 2021.

\bibitem{costa16hierarchical}
M.~{Leandro Costa}, C.~A. {Araujo Padilha}, J.~{Dantas Melo}, and A.~{Duarte
  Doria Neto}.
\newblock Hierarchical reinforcement learning and parallel computing applied to
  the k-server problem.
\newblock {\em IEEE Latin America Transactions}, 14(10):4351--4357, 2016.

\bibitem{Lindermayr2020}
Alexander Lindermayr.
\newblock Learning-augmented online algorithms for the 2-server problem on the
  line and generalizations.
\newblock Master's thesis, University of Bremen, Germany, 2020.

\bibitem{DBLP:journals/eswa/LinsNM19}
Ramon Augusto~Sousa Lins, Adri{\~{a}}o Duarte~D{\'{o}}ria Neto, and
  Jorge~Dantas de~Melo.
\newblock Deep reinforcement learning applied to the \emph{k}-server problem.
\newblock {\em Expert Syst. Appl.}, 135:212--218, 2019.

\bibitem{lu2021generalized}
Pinyan Lu, Xuandi Ren, Enze Sun, and Yubo Zhang.
\newblock Generalized sorting with predictions.
\newblock In {\em Symposium on Simplicity in Algorithms ({SOSA})}, pages
  111--117. {SIAM}, 2021.

\bibitem{lykouris2018caching}
Thodoris Lykouris and Sergei Vassilvitskii.
\newblock Competitive caching with machine learned advice.
\newblock In {\em {ICML}}, volume~80 of {\em Proceedings of Machine Learning
  Research}, pages 3302--3311. {PMLR}, 2018.

\bibitem{MahdianNS07}
Mohammad Mahdian, Hamid Nazerzadeh, and Amin Saberi.
\newblock Allocating online advertisement space with unreliable estimates.
\newblock In {\em {EC}}, pages 288--294. {ACM}, 2007.

\bibitem{ManasseMS88}
Mark~S. Manasse, Lyle~A. McGeoch, and Daniel~Dominic Sleator.
\newblock Competitive algorithms for on-line problems.
\newblock In {\em {STOC}}, pages 322--333. {ACM}, 1988.

\bibitem{manasse1990server}
Mark~S. Manasse, Lyle~A. McGeoch, and Daniel~Dominic Sleator.
\newblock Competitive algorithms for server problems.
\newblock {\em J. Algorithms}, 11(2):208--230, 1990.

\bibitem{munoz2017revenue}
Andres~Mu{\~{n}}oz Medina and Sergei Vassilvitskii.
\newblock Revenue optimization with approximate bid predictions.
\newblock In {\em {NIPS}}, pages 1858--1866, 2017.

\bibitem{mitzenmacher2018bloom}
Michael Mitzenmacher.
\newblock A model for learned bloom filters and optimizing by sandwiching.
\newblock In {\em NeurIPS}, pages 462--471, 2018.

\bibitem{mitzenmacher2020scheduling}
Michael Mitzenmacher.
\newblock Scheduling with predictions and the price of misprediction.
\newblock In {\em {ITCS}}, volume 151 of {\em LIPIcs}, pages 14:1--14:18.
  Schloss Dagstuhl - Leibniz-Zentrum f{\"{u}}r Informatik, 2020.

\bibitem{purohit2018improving}
Manish Purohit, Zoya Svitkina, and Ravi Kumar.
\newblock Improving online algorithms via {ML} predictions.
\newblock In {\em {NeurIPS}}, pages 9684--9693, 2018.

\bibitem{rohatgi2020caching}
Dhruv Rohatgi.
\newblock Near-optimal bounds for online caching with machine learned advice.
\newblock In {\em {SODA}}, pages 1834--1845. {SIAM}, 2020.

\bibitem{shalevB14ML}
Shai Shalev{-}Shwartz and Shai Ben{-}David.
\newblock {\em Understanding Machine Learning - From Theory to Algorithms}.
\newblock Cambridge University Press, 2014.

\bibitem{SleatorT85}
Daniel~Dominic Sleator and Robert~Endre Tarjan.
\newblock Amortized efficiency of list update and paging rules.
\newblock {\em Commun. {ACM}}, 28(2):202--208, 1985.

\bibitem{Valiant84}
Leslie~G. Valiant.
\newblock A theory of the learnable.
\newblock {\em Commun. {ACM}}, 27(11):1134--1142, 1984.

\bibitem{vapnik1971}
VN~Vapnik and A~Ya Chervonenkis.
\newblock On the uniform convergence of relative frequencies of events to their
  probabilities.
\newblock {\em Theory of Probability \& Its Applications}, 16(2):264--280,
  1971.

\bibitem{wangLW20}
Shufan Wang, Jian Li, and Shiqiang Wang.
\newblock Online algorithms for multi-shop ski rental with machine learned
  advice.
\newblock In {\em {NeurIPS}}, 2020.

\bibitem{wei2020better}
Alexander Wei.
\newblock Better and simpler learning-augmented online caching.
\newblock In {\em {APPROX/RANDOM}}, volume 176 of {\em LIPIcs}, pages
  60:1--60:17. Schloss Dagstuhl - Leibniz-Zentrum f{\"{u}}r Informatik, 2020.

\bibitem{wei2020tradeoff}
Alexander Wei and Fred Zhang.
\newblock Optimal robustness-consistency trade-offs for learning-augmented
  online algorithms.
\newblock In {\em NeurIPS}, 2020.

\bibitem{zhang2020wfbound}
Wenming Zhang and Yongxi Cheng.
\newblock A new upper bound on the work function algorithm for the k-server
  problem.
\newblock {\em J. Comb. Optim.}, 39(2):509--518, 2020.

\end{thebibliography}

\newpage
\appendix

\section{Proofs for \Cref{sec:roadmap} }\label{app:lDCrobust}

\lemmatightness*
\begin{proof}
	We give two separate instances for consistency and robustness.
	\begin{enumerate}[(i)]
		\item
		Consider~$k$ servers initially at positions 0,~$1$,~$-1$,~$2$,~$-2$, \dots~and the request sequence of length~$k+1$ at positions~$0.5$,~$0$,~$1$,~$-1$,~$2$,~$-2$, \dots. There is a solution of cost $1$ that only moves the server that is initially at~$0$.
		
		\ldc serves the first request by moving the optimal server from $0$ to $0.5$ and additionally the one from~$1$ to~$1-\lambda/2$. With the second request, the first server is moved back to $0$, having moved a total distance of $1$, and the server from $-1$ moves to $-1+\lambda/2$. For the third request, the server from original position $1$ returns to this position, etc.
		Each server moves back to its initial position~$i$ after moving a total distance of~$\lambda^{|i|}$. Repeating this example gives the lower bound on the consistency.
		\item Consider~$k$ servers initially at positions~$\rob(i)=\sum_{i=0}^{k-1} \lambda^{-i}$, for~$i\in\{1,\ldots,k\}$, and the request sequence of length~$k+1$ at positions~$0$,~$\rob(1)$,~$\rob(2)$, \dots,~$\rob(k)$. There is a solution of cost $2$ that only moves the server that is initially at~$1$. Consider predictions corresponding always to the rightmost server at the highest position.

		\ldc serves the second request by moving both servers from 0 and~$\rob(2)$ to~$\rob(1)$ as the closest server moves by a distance of 1 and the furthest server, which is predicted, moves by a distance of~$1/\lambda$. Similarly, for each request except the last one, both servers neighboring the request end up serving the request simultaneously. So the~$i$-th server moves by a total distance of~$2/\lambda^{i-1}$.
   Repeating this example gives the lower bound on the robustness.
   \qedhere
	\end{enumerate}
\end{proof}

\section{Proofs for \Cref{sec:k-server}}\label{app:upper-bounds}

\subparagraph{The Consistency Bound}

\LemGenConsistency*

{Before proving \Cref{lemma:general-consistency-upper-bound}, we need a few preliminary results. Let~$I$ be an arbitrary instance. Note that~$\lambda = 0$ implies~$\alg(I) \leq \ftp(I)$ as \ldc can only shortcut \ftp's moves, so we now assume that~$\lambda \in (0,1]$.

\begin{observation}\label{claim:consistency-technical-eqs}
     {For every~$k>3$, we have~$\frac{\cons(k)}{\lambda} = \cons(k - 2) + \frac{1}{\lambda} + 1$.}%
\end{observation}
\begin{proof}
    {We prove this statement depending on the parity of~$k$}. If~$k$ is odd,~$k-2$ is also odd. {By definition of~$\alpha(k)$,}
    \[
        1 + \frac{1}{\lambda} + \alpha(k-2) = \frac{1}{\lambda} + 2 + 2 \sum_{i=1}^{(k-3)/2} \lambda^i = \frac{1}{\lambda} + 2 \sum_{i=1}^{(k-1)/2} \lambda^{i-1} = \frac{\alpha(k)}{\lambda}.
    \]
    If~$k$ is even,~$k-2$ is also even, {and we conclude}
    \begin{align*}
        1 + \frac{1}{\lambda} + \alpha(k-2) &= \frac{1}{\lambda} + 2 + 2 \sum_{i=1}^{(k-2)/2-1} \lambda^i + \lambda^{(k-2)/2} \\
		&= \frac{1}{\lambda} + 2 \sum_{i=1}^{k/2-1} \lambda^{i-1} + \lambda^{(k-2)/2} \\
		&= \frac{\alpha(k)}{\lambda}. 
	\end{align*}
\end{proof}

We defined our potential function~$\Phi$ as follows.
Let~$s_1, \ldots, s_k$ be the servers of \ldc and let~$x_1', \ldots, x_k'$ be the servers of \ftp.
For~$1 \leq i < j \leq k$ and~$\ell = \min\{j-i, k-(j-i)\}-1$ we define~$\delta_{ij} = \lambda^{\ell}$. Then,
\[
    \Phi =  \underbrace{\frac{\cons(k)}{\lambda} \cdot\sum_{i=1}^k \abs{s_i - x_i'}}_\Psi  + \underbrace{\sum_{i < j} \delta_{ij} \abs{s_i - s_j}.}_\Theta
\]
\begin{figure}[tb]
	\centering
	\begin{adjustbox}{width=0.7\linewidth}
		\begin{tikzpicture}[scale=0.75]
			\node[circle, draw, thick, minimum size=8mm, inner sep=1pt, outer sep=0pt] (s1) at (0,0) {$s_1$};
			\node[circle, draw, thick, minimum size=8mm, inner sep=1pt, outer sep=0pt] (s2) at (2.5,0) {$s_2$};
			\node[circle, draw, thick, minimum size=8mm, inner sep=1pt, outer sep=0pt] (s3) at (5,0) {$s_3$};
			\node[circle, draw, thick, minimum size=8mm, inner sep=1pt, outer sep=0pt] (s4) at (7.5,0) {$s_4$};
			\node at (10,0) {$\cdots$};
			\node[circle, draw, thick, minimum size=8mm, inner sep=1pt, outer sep=0pt] (sk-2) at (12.5,0) {$s_{k-2}$};
			\node[circle, draw, thick, minimum size=8mm, inner sep=1pt, outer sep=0pt] (sk-1) at (15,0) {$s_{k-1}$};
			\node[circle, draw, thick, minimum size=8mm, inner sep=1pt, outer sep=0pt] (sk) at (17.5,0) {$s_k$};
		
			\newcommand{\bendparam}{40}
			
			\draw[line width=3.3pt] (s1) edge node[above, pos=0.5, xshift=3pt] {$\ell = 0$} (s2);
			\draw[line width=2.6pt, algBlue] (s1) edge[bend left=\bendparam] node[above, pos=0.9, xshift=8pt] {$\ell = 1$} (s3);
			\draw[line width=1.9pt, algOrange] (s1) edge[bend left=\bendparam] node[above, pos=0.9, xshift=8pt] {$\ell = 2$} (s4);
			\draw[line width=1.9pt, algOrange] (s1) edge[bend left=\bendparam] node[above, pos=0.95, xshift=8pt] {$\ell = 2$} (sk-2);
			\draw[line width=2.6pt, algBlue] (s1) edge[bend left=\bendparam] node[above, pos=0.96, xshift=9pt] {$\ell = 1$} (sk-1);
			\draw[line width=3.3pt] (s1) edge[bend left=\bendparam] node[above, pos=0.97, xshift=10pt] {$\ell = 0$} (sk);
		
			\draw[line width=3.3pt] (s2) edge node[below, pos=0.5, xshift=3pt] {$\ell = 0$} (s3);
			\draw[line width=2.6pt, algBlue] (s2) edge[bend right=\bendparam] node[below, pos=0.9, xshift=8pt] {$\ell = 1$} (s4);
			\draw[line width=1.2pt, algGreen] (s2) edge[bend right=\bendparam] node[below, pos=0.95, xshift=8pt] {$\ell = 3$} (sk-2);
			\draw[line width=1.9pt, algOrange] (s2) edge[bend right=\bendparam] node[below, pos=0.96, xshift=8pt] {$\ell = 2$} (sk-1);
			\draw[line width=2.6pt, algBlue] (s2) edge[bend right=\bendparam] node[below, pos=0.97, xshift=8pt] {$\ell = 1$} (sk);
		\end{tikzpicture}
	\end{adjustbox}
	\caption{Visualization of all incident~$\delta_{ij}$-weights of the servers~$s_1$ and~$s_2$. The thickness (resp.\ color) of an arc indicates the influence of the corresponding distance in~$\Phi$.}
\end{figure}
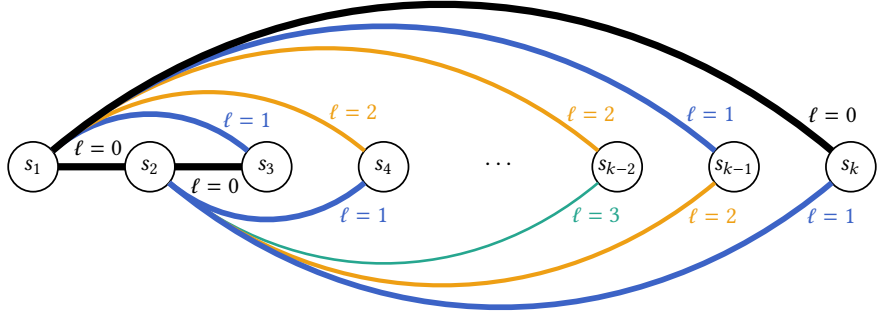
The analysis of \ldc requires evaluating the evolution of $\Phi$ after each request.   
The following lemma characterizes %
how a move of \ldc influences~$\Theta$. %
\begin{lemma}\label{lemma:consistency-theta-change}
    Let~$i \leq \floor{k/2}$. If~$s_i$ moves from~$p$ to~$p + x$,~$\Theta$ changes
by 
    \[
        (-x) \cdot \left( 1 + \cons(k-2) - 2\sum_{j=0}^{i-2} \lambda^j \right).
    \]
\end{lemma}
\begin{proof}
    Assume w.l.o.g. that $x=-1$, that is, server $s_i$ moves one unit to the left. Consider servers~$s_j$, $s_j'$ such that $j' + \ell = i = j - \ell$ for some~$1 \leq \ell \leq i-1$. Since~$\delta_{ij'} = \delta_{ij}$ we observe that the changes to the terms~$\delta_{ij} \abs{s_i - s_j}$ and~$ \delta_{ij'} \abs{s_i - s_j'}$ of~$\Theta$ cancel out. Hence, as~$i \leq \floor{k/2}$, the change of~$\Theta$ due to the move of~$s_i$ is equal to~$\sum_{j=2i}^k \delta_{ij}$.
    We now prove the statement depending on the parity of~$k$.
    \begin{enumerate}[(i)]
        \item If~$k$ is odd,~$k - 2i + 1$ is even. By definition,
              \begin{align*}
                  \sum_{j=2i}^k \delta_{ij} & = \sum_{j=2i}^k \lambda^{\min\{j-i, k-(j-i)\} - 1} = 2\sum_{j=2i}^{(k-1)/2 + i} \lambda^{j - i - 1} = 2\sum_{j=i-1}^{(k-1)/2 - 1} \lambda^j \\
                                            & = 2 + 2\sum_{j=1}^{(k-3)/2} \lambda^j - 2\sum_{j=0}^{i-2} \lambda^j = 1 + \cons(k-2) - 2\sum_{j=0}^{i-2} \lambda^j.
              \end{align*}
        \item If~$k$ is even,~$k - 2i + 1$ is odd, and there is a single term where the minimum in the definition of~$\delta_{ij}$ in achieved by both conditions. Hence,
              \begin{align*}
                  \sum_{j=2i}^k \lambda^{\min\{j-i, k-(j-i)\} - 1} & = 2\sum_{j=2i}^{k/2 - 1 + i} \lambda^{j - i - 1} + \lambda^{k/2-1} = 2\sum_{j=i-1}^{k/2 - 2} \lambda^j + \lambda^{k/2-1}            \\
                                                                   & = 2 + 2\sum_{j=1}^{k/2-2} \lambda^j + \lambda^{k/2-1} - 2\sum_{j=0}^{i-2} \lambda^j \\
																   &= 1 + \cons(k-2) - 2\sum_{j=0}^{i-2} \lambda^j.\qedhere 
              \end{align*}
    \end{enumerate}
\end{proof}

\begin{proof}[Proof of~\Cref{lemma:general-consistency-upper-bound}]
    Suppose that the next request appears. First \ftp moves some server~$x'_i$ towards the request, and the distance to~$s_i$ increases by at most~$\Delta \ftp$. Since this move only affects~$\Psi$,~$\Delta \Phi \leq \cons(k) / \lambda \cdot \Delta \ftp$.
    Second \ldc moves. We distinguish whether the request is between two servers or not, and assert for both cases~$\Delta \Phi \leq - 1 / \lambda \cdot \Delta \alg$.
    \begin{enumerate}[(a)]
        \item Let the request be located w.l.o.g on the left of~$s_1$. Thus,~$s_1$ moves towards it and charges cost~$\Delta \alg$. The fact that some server~$x_j'$ must already be on~$r_t$ implies with~\Cref{lemma:consistency-theta-change} that
              \[
                  \Delta \Phi \leq -\frac{\cons(k)}{\lambda}\Delta \alg + \left( 1 + \cons(k-2) - 2\sum_{j=0}^{1 - 2} \lambda^j \right) \Delta \alg.
              \]
              Rearranging and using~\Cref{claim:consistency-technical-eqs} gives the claimed bound, that is
              \[
                  \left( -\cons(k - 2) - \frac{1}{\lambda} - 1 + 1 + \cons(k-2) \right) \Delta \alg = -\frac{1}{\lambda} \Delta \alg.
              \]
        \item~Let the request be between~$s_i$ and~$s_{i+1}$.
              Assume w.l.o.g. that \ftp serves it with~$x'_j$ and~$j \leq i$.
              For ease of exposition, we assume that~$s_i$ travels distance~$1$ and~$s_{i+1}$ distance~$\lambda$. Hence,~$\Delta \alg = 1 + \lambda$.
              Since~$j \leq i$, we know that~$x_i'$ must be located on the right of~$x_j'$.
              Hence, the distance between~$s_i$ and~$x_i'$ decreases by~$1$, but the distance between~$s_{i+1}$ and~$x_{i+1}'$ increases by at most~$\lambda$. Thus,~$\Delta \Psi \leq \alpha(k)/\lambda \cdot (\lambda - 1)$.
              The change of~$\Theta$ is clearly bounded from above by the case where~$s_i$ moves distance~$\lambda$ and~$s_{i+1}$ moves distance~$1$ for~$i+1 \leq \floor{k/2}$. Combining~\Cref{lemma:consistency-theta-change} for both servers gives
              \begin{align*}
                  \Delta \Theta & = 1 + \cons(k-2) - 2\sum_{j=0}^{i-1}\lambda^{j} - \lambda \left( 1 + \cons(k-2) - 2\sum_{j=0}^{i-2}\lambda^{j} \right) \\
                                & = -1 - \lambda  + \cons(k-2) - \lambda \cons(k-2).
              \end{align*}
              Using this and~\Cref{claim:consistency-technical-eqs}, we can bound the increase of the potential by
              \begin{align*}
                  \Delta \Phi & \leq \frac{\cons(k)}{\lambda}(\lambda - 1) -1 - \lambda  + \cons(k-2) - \lambda \cons(k-2)    \\
                              & =- \frac{\cons(k)}{\lambda}  + \cons(k-2) %
							  = -1 - \frac{1}{\lambda} \\
							  &= -\frac{1}{\lambda} \Delta \alg. \qedhere
              \end{align*}
    \end{enumerate}
\end{proof}

\subparagraph{The Robustness Bound}

\LemGenRobustness*

We start by defining a potential function~$\Phi$. 
Let~$\lambda \in (0,1]$. Fix~$k$, let~$\beta = \beta(k)=\sum_{i=0}^{k-1}\lambda^{-i}$, and let~$s_1, \ldots, s_k$ be the servers of \ldc and let~$x_1, \ldots, x_k$ be the servers of an optimal solution. The potential function~is
\[
	\Phi = \underbrace{\rob \gamma \left(\sum_{i=1}^k \omega_i \abs{s_i - x_i} \right)}_{\Psi} + \underbrace{\sum_{i<j} \delta_{ij} \abs{s_i - s_j}}_{\Theta}.
\]
We specify the weights in this function as follows. For a pair of servers~$s_i,s_j$ with~$1 \leq i < j \leq k$, let~$\ell = \min\{j-i, k-(j-i)\}-1$ and
\[
	\delta_{ij} = \zeta_\ell 
= \frac{\lambda^\ell + \lambda^{k - 2 -\ell}}{1 + \lambda^{k-2}}.
\]

The intuition of the weights in the spreadness part $\Theta$ is the same as %
in the consistency potential function. However, the new 
weights~$\omega_i$ in the matching part $\Psi$ (defined below) require 
the more complex weights $\delta_{ij}$ compared to the simpler~$\lambda^{\ell}$ weights. 

Further, we define $d_{\ceil{k/2}} = 0$ if $k$ is odd and for all~$1 \leq i \leq \floor{k/2}$ let
\[
	d_i = d_{k+1-i} = \frac{2}{1 + \lambda^{k-2}} \sum_{\ell=i-1}^{k-1-i} \lambda^\ell.
\]

Let~$\gamma = d_1 / (\beta - 1)$,~$\omega_1 = \omega_k = 1$ and for~$2 \leq i \leq \ceil{k/2}$ we define the server-individual weights

	\begin{align*}
		\omega_i = \omega_{k+1-i} = \begin{dcases}
			\frac{2 \lambda \sum_{j=1}^{i/2-1} d_{2j} - 2 \sum_{j=1}^{i/2-1} d_{2j+1} + \lambda d_i + (2 + \lambda) \gamma}{\rob \gamma \lambda} & \text{ if } i \text{ is even, and} \\
			\frac{2 \lambda \sum_{j=1}^{(i-1)/2} d_{2j} - 2 \sum_{j=1}^{(i-3)/2} d_{2j+1} - d_i + \gamma}{\rob \gamma}                           & \text{ if } i \text{ is odd.}      \\
		\end{dcases}
	\end{align*}

{This finishes the definition of the potential function~$\Phi$. %
{To prove a robustness guarantee for \ldc, we} show bounds on the change of~$\Phi$ when the algorithms {(\ldc and \opt)} move their servers. {To that end, }%
several preliminary results {will become handy}. We first observe that the values~$d_1,\ldots,d_k$ correlate with the change of~$\Theta$ when \ldc moves a server.}
\begin{observation}\label{claim:scale-of-theta}
	Let~$i \leq \floor{k/2}$. If server~$s_i$ moves from~$p$ to~$p + x$,~$\Theta$ changes by~$(-x) \cdot d_i$.
\end{observation}
\begin{proof}
	Assume w.l.o.g. that~$x = -1$, that is, the server~$s_i$ moves one unit to the left. Consider servers~$s_j, s_j'$ such that $j' + \ell = i = j - \ell$ for some~$1 \leq \ell \leq i-1$. Since~$\delta_{ij'} = \delta_{ij}$ we observe that the changes to the terms~$\delta_{ij} \abs{s_i - s_j}$ and~$ \delta_{ij'} \abs{s_i - s_j'}$ of~$\Theta$ cancel out. Hence, as~$i \leq \floor{k/2}$, it suffices to consider the distances of~$s_i$ to servers~$s_j$ with~$j \geq 2i$. Therefore, 
	\[
		\Delta \Theta = \sum_{j=2i}^k \delta_{ij} = \begin{cases}
			\sum_{\ell=i-1}^{k/2-2} 2 \zeta_\ell + \zeta_{k/2-1} & \text{ if } k \text{ is even, and} \\
			\sum_{\ell=i-1}^{(k-3)/2} 2 \zeta_\ell               & \text{ if } k \text{ is odd}.
		\end{cases}
	\]
	The definition of~$\zeta_\ell$ implies that this is indeed equal to~$d_i$.
\end{proof}

{
Next, we give several algebraic transformations of~$\gamma$. 
\begin{lemma}
	The following statements are true:
	\vspace*{-.5em}
	\begin{thmlist}
		\item\label[lemma]{obs:gamma} $\gamma = 2 \lambda^{k-1} / (1 + \lambda^{k-2})$.
		\item For all~$1 \leq i \leq \floor{k/2}$, it holds~$(1+\lambda)\gamma = \lambda^{i+1} d_i - \lambda^i d_{i+1}$.\label[lemma]{claim:d-gamma}
		\item If~$k$ is even, it holds~$\gamma = \lambda d_1 + (1+\lambda) \sum_{j=2}^{k/2} {(-1)}^{j-1} d_{j}$.\label[lemma]{claim:gamma-alternating}
	\end{thmlist}
\end{lemma}
}
\begin{proof}
	\begin{enumerate}[(i)]
		\item Since
		\[
			\lambda^{k-1} (\rob - 1) = \lambda^{k-1} \sum_{\ell=1}^{k-1} \lambda^{-\ell} = \sum_{\ell=0}^{k-2} \lambda^\ell,
		\]
		we conclude by the definition of~$\gamma$ and~$d_1$ that
		\[
			\gamma = \frac{d_1}{\rob-1} = \frac{2}{(1 + \lambda^{k-2})(\rob-1)}	\sum_{\ell=0}^{k-2} \lambda^\ell = \frac{2}{1 + \lambda^{k-2}}	\lambda^{k-1}.
		\]
		\item Simplifying the right-hand side gives
		\begin{align*}
			\lambda^{i+1} d_i - \lambda^i d_{i+1}
			 &                             
			  = \frac{2 \lambda^i}{1 + \lambda^{k-2}} \left( \lambda \sum_{\ell=i-1}^{k-1-i} \lambda^\ell -\sum_{\ell=i}^{k-2-i} \lambda^\ell \right) 
			  = \frac{2 \lambda^i}{1 + \lambda^{k-2}} \left( \sum_{\ell=i}^{k-i} \lambda^\ell -\sum_{\ell=i}^{k-2-i} \lambda^\ell \right)             \\
			 & = \frac{2 \lambda^i}{1 + \lambda^{k-2}} \left( \lambda^{k-i} + \lambda^{k-i-1} \right)                                                  
			  = 2 \left( \frac{\lambda^{k} + \lambda^{k-1}}{1 + \lambda^{k-2}} \right)                                        \\                        
			 & = (1+\lambda) \left( \frac{2 \lambda^{k-1}}{1 + \lambda^{k-2}} \right).
		\end{align*}
		Then,~\Cref{obs:gamma} concludes the proof.
		\item Assume that~$k$ is even. The right-hand side is equal to
		\[
			{(-1)}^{k/2-1} \lambda d_{k/2} + \sum_{j=1}^{k/2-1} {(-1)}^{j-1} (\lambda d_j - d_{j+1}).
		\]
		By~\Cref{claim:d-gamma},
		\[
			{(-1)}^{k/2-1} \lambda d_{k/2} + (1+\lambda) \sum_{j=1}^{k/2-1} {(-1)}^{j-1} \frac{\gamma}{\lambda^j},
		\]
		which is equal to
		\[
			\frac{1}{\lambda^{k/2-1}} \left( {(-1)}^{k/2-1} \lambda^{k/2} d_{k/2} + \sum_{j=0}^{k/2-2} {(-1)}^{k/2-j} (\lambda^{j} + \lambda^{j+1}) \gamma \right).
		\]
		We proceed by applying a telescoping sum argument. Since~$k/2-(k/2-2) = 2$, the last term of the sum~$\lambda^{k/2-1}\gamma$ is positive. Similarly, the first term~$\lambda^0 \gamma$ has the same sign as~${(-1)}^{k/2-0} = -{(-1)}^{k/2-1}$. The remaining terms of the sum cancel out. Thus, it remains
		\begin{align*}
			\frac{1}{\lambda^{k/2-1}} \left( \lambda^{k/2-1} \gamma + {(-1)}^{k/2-1} (\lambda^{k/2} d_{k/2} - \gamma) \right).
		\end{align*}
		By definition,~$d_{k/2} = 2 \lambda^{k/2-1} / (1 + \lambda^{k-2})$. Hence,~$\lambda^{k/2} d_{k/2}$ is equal to~$\gamma$ by~\Cref{obs:gamma}. We conclude that the expression is indeed equal to~$\gamma$.
	\end{enumerate}	
\end{proof}

These preliminary results enable us to prove two more involved observations {about the weights chosen for our potential function. The proofs are deferred to~\Cref{app:upper-bounds}.} %
The first observation %
is important for all cases where a request appears between two servers. Recall the definition of~$\Phi$. If~$s_i$ moves with speed~$\lambda$ and $s_{i+1}$ with speed~$1$, the changes to~$\Psi$ (increase or decrease) are scaled by~$\beta \gamma \lambda \omega_i$ regarding $s_i$ and $\beta \gamma \omega_{i+1}$ regarding~$s_{i+1}$. If~$i$ is even, we can easily use the definition of $\omega$, since the denominators cancel. However, if~$i$ is odd, we use the following alternative representation of the~$\omega$-weights.%
\begin{observation}\label{lemma:omegas-are-equal}
	For~$2 \leq i \leq \ceil{k/2}$,~$\omega_i$ is equal to
	\begin{align*}
		\begin{dcases}
			\frac{2 \lambda \sum_{j=1}^{i/2} d_{2j-1} - 2 \sum_{j=1}^{i/2-1} d_{2j} - d_i - \gamma}{\rob \gamma}                               & \text{ if } i \text{ is even, and} \\
			\frac{2 \lambda \sum_{j=1}^{(i-1)/2} d_{2j-1} - 2 \sum_{j=1}^{(i-1)/2} d_{2j} + \lambda d_i + \lambda \gamma}{\rob \gamma \lambda} & \text{ if } i \text{ is odd.}      \\
		\end{dcases}
	\end{align*}
\end{observation}
\begin{proof}
	{We first note that for every~$1 \leq j \leq \floor{k/2}-1$, applying~\Cref{claim:d-gamma} with~$j$ and~$j+1$ yields
	\begin{equation}
		\lambda d_j - d_{j+1} = \frac{1+\lambda}{\lambda^j}\gamma = \lambda^2 d_{j+1} - \lambda d_{j+2}.\label{eq:3-consec-d}
	\end{equation}
	}
	We now prove the statement separately for all even and all odd values of~$2 \leq i \leq \ceil{k/2}$ by induction.

		As induction base for the even case, we first prove the claim for~$i=2$. Indeed,
		      \[
			      \omega_2
			      = \frac{\lambda d_2 + (2 + \lambda) \gamma}{\rob \gamma \lambda}
			      = \frac{d_2 + (2 + \lambda) \gamma / \lambda}{\rob \gamma}
			      = \frac{2 \lambda d_1 - d_2 - \gamma}{\rob \gamma}.
		      \]
		      Note that the last equality derives from~\Cref{claim:d-gamma}.
		      Now assume that~$i > 2$ is even. The induction hypothesis for~$i-2$ yields in this case
		      \begin{equation}
				  \beta \gamma \lambda \cdot \omega_{i-2} = 2 \lambda^2 \sum_{j=1}^{i/2-1} d_{2j-1} - 2 \lambda \sum_{j=1}^{i/2-2} d_{2j} - \lambda d_{i-2} - \lambda \gamma. \label{eq:induction-hypo-even}
		      \end{equation}
		      We want to prove that~$\beta \gamma \lambda  \cdot \omega_i$ is equal to
		      \[
			      2 \lambda^2 \sum_{j=1}^{i/2} d_{2j-1} - 2 \lambda \sum_{j=1}^{i/2-1} d_{2j} - \lambda d_i - \lambda \gamma,
		      \]
		      which can be rearranged to
		      \[
			      2 \lambda^2 \sum_{j=1}^{i/2-1} d_{2j-1} - 2 \lambda \sum_{j=1}^{i/2-2} d_{2j} - \lambda d_{i-2} - \lambda \gamma - \lambda d_{i} - \lambda d_{i-2} + 2 \lambda^2 d_{i-1}.
			  \]
			  Replacing the right side of~\eqref{eq:induction-hypo-even} in the above expression yields
			  \[
				\beta \gamma \lambda \cdot \omega_{i-2} - \lambda d_{i} - \lambda d_{i-2} + 2 \lambda^2 d_{i-1}.
			  \]
		      Since~\eqref{eq:3-consec-d} gives~$2(1+\lambda^2) d_{i-1} = 2 \lambda (d_{i-2} + d_{i})$, and by the definition of~$\omega_{i-2}$, this can be rewritten to
			  \begin{align*}  
				&2 \lambda \sum_{j=1}^{i/2-2} d_{2j} - 2 \sum_{j=1}^{i/2-2} d_{2j+1} + \lambda d_{i-2} + (2 + \lambda) \gamma - 2d_{i-1} + \lambda d_{i-2} + \lambda d_i \\
				&= 2 \lambda \sum_{j=1}^{i/2-1} d_{2j} - 2 \sum_{j=1}^{i/2-1} d_{2j+1} + (2 + \lambda) \gamma + \lambda d_i,
			 \end{align*}
		      which is indeed equal to~$\beta \gamma \lambda \cdot \omega_i$ by definition.

		As induction base for the odd case, we start by proving the claim for~$i=3$, that is
		      \[
			      \omega_3 = \frac{2 \lambda d_2 - d_3 + \gamma}{\beta \gamma} = \frac{2 \lambda^2 d_2 - \lambda d_3 + \lambda \gamma}{\beta \gamma \lambda}  = \frac{2\lambda d_1 - 2d_2 + \lambda d_3 + \lambda \gamma}{\rob \gamma \lambda}.
		      \]
		      In the last equality we used that~$2(1 + \lambda^2) d_2 = 2\lambda (d_1 + d_3)$ by~\eqref{eq:3-consec-d}. Now assume that~$i > 3$ is odd.
		      By induction hypothesis for~$i-2$,
		      \begin{equation}
				\beta \gamma \lambda \cdot \omega_{i-2} = 2 \lambda \sum_{j=1}^{(i-1)/2-1} d_{2j-1} - 2 \sum_{j=1}^{(i-1)/2-1} d_{2j} + \lambda d_{i-2} + \lambda \gamma. \label{eq:induction-hypo-odd}
		      \end{equation}
		      Consider the claimed expression for~$\beta \gamma \lambda \cdot \omega_i$, that is
		      \[
			      2 \lambda \sum_{j=1}^{(i-1)/2} d_{2j-1} - 2 \sum_{j=1}^{(i-1)/2} d_{2j} + \lambda d_i + \lambda \gamma,
		      \]
		      which we can rearrange to
		      \[
			      2 \lambda \sum_{j=1}^{(i-1)/2-1} d_{2j-1} - 2 \sum_{j=1}^{(i-1)/2-1} d_{2j} + \lambda d_{i-2} + \lambda \gamma + \lambda d_{i-2} + \lambda d_i - 2 d_{i-1}.
			  \]
			  Replacing the right side of~\eqref{eq:induction-hypo-odd} in the above expression gives
			  \[
				\beta \gamma \lambda \cdot \omega_{i-2} + \lambda d_{i-2} + \lambda d_i - 2 d_{i-1}.
			  \]
		      Noting that~\eqref{eq:3-consec-d} gives~$2(1+\lambda^2) d_{i-1} = 2 \lambda (d_{i-2} + d_{i})$ yields together with the definition of~$\omega_{i-2}$ the equivalent expression
		      \begin{align*}
				&2 \lambda^2 \sum_{j=1}^{(i-1)/2-1} d_{2j} - 2 \lambda \sum_{j=1}^{(i-3)/2-1} d_{2j+1} - \lambda d_{i-2} + \lambda \gamma + 2 \lambda^2 d_{i-1} - \lambda d_{i-2} - \lambda d_i \\
				&= 2 \lambda^2 \sum_{j=1}^{(i-1)/2} d_{2j} - 2 \lambda \sum_{j=1}^{(i-3)/2} d_{2j+1} - \lambda d_{i} + \lambda \gamma.
			  \end{align*}
		      Since this is by definition equal to~$\beta \gamma \lambda \cdot \omega_i$, we can also conclude this case.
\end{proof}

The second observation is an upper and a lower bound of the~$\omega$-weights regardless of the corresponding server. The lower bound is necessary to show that~$\Phi \geq 0$, while we use the upper bound to give an easy upper bound on the increase of the potential when the optimal solution moves, independently of its chosen server. %
\begin{observation}\label{claim:omega-bound}
	The values~$\omega_1,\ldots,\omega_k$ are at least~$0$ and at most~$1$.
\end{observation}
\begin{proof}
	By definition,~$\omega_1 = \omega_k = 1$. We now show this property for~$\omega_i$ depending on whether~$2 \leq i \leq k-1$ is even or odd.

		Assume that~$i$ is even. By definition, the numerator of~$\omega_i$ is equal to
		      \[
			      2 \lambda \sum_{j=1}^{i/2-1} d_{2j} - 2 \sum_{j=1}^{i/2-1} d_{2j+1} + \lambda d_i + (2 + \lambda) \gamma.
		      \]
			  Using the definition of~$d_i$ and~\Cref{obs:gamma} gives
		      \begin{align*}
			       & \frac{2}{1 + \lambda^{k-2}} \left( \lambda \sum_{j=1}^{i/2-1} \sum_{\ell=2j-1}^{k-1-2j} 2 \lambda^\ell - \sum_{j=1}^{i/2-1} \sum_{\ell=2j}^{k-2-2j} 2 \lambda^\ell + \lambda \sum_{\ell=i-1}^{k-1-i} \lambda^\ell \right) + (2+\lambda) \gamma \\
			       & = \frac{2}{1 + \lambda^{k-2}} \left(\sum_{j=1}^{i/2-1} 2 (\lambda^{k-2j} + \lambda^{k-1-2j})  + \sum_{\ell=i}^{k-i} \lambda^\ell \right) + (2+\lambda) \gamma                                                                                  \\
			       & = \frac{2}{1 + \lambda^{k-2}} \left(\sum_{\ell=k-i+1}^{k-2} 2 \lambda^\ell  + \sum_{\ell=i}^{k-i} \lambda^\ell \right) + (2+\lambda) \gamma                                                                                                    \\
			       & = \frac{2}{1 + \lambda^{k-2}} \left( \sum_{\ell=k-i+1}^{k-1} 2 \lambda^\ell  + \sum_{\ell=i}^{k-i} \lambda^\ell\right) + \lambda \gamma
		      \end{align*}
		      Since~$\rob \gamma \lambda = \lambda \gamma + \lambda d_1 \geq 0$, we conclude that~$\omega_i \geq 0$. Further, using the fact that~$\sum_{\ell=k-i+1}^{k-1} \lambda^\ell \leq \sum_{\ell=1}^{i-1} \lambda^\ell$ yields
		      \[
			      \frac{2}{1 + \lambda^{k-2}} \left( \sum_{\ell=k-i+1}^{k-1} 2 \lambda^\ell  + \sum_{\ell=i}^{k-i} \lambda^\ell\right) + \lambda \gamma \leq \frac{2}{1 + \lambda^{k-2}} \sum_{\ell=1}^{k-1} \lambda^{\ell} + \lambda \gamma = \lambda d_1 + \lambda \gamma,
		      \]
		      and we conclude that~$\omega_i \leq 1$.

		Assume that~$i$ is odd. By definition, the numerator of~$\omega_i$ is equal to
		      \[
			      2 \lambda \sum_{j=1}^{(i-1)/2} d_{2j} - 2 \sum_{j=1}^{(i-3)/2} d_{2j+1} - d_i + \gamma.
		      \]
		     Using definitions gives
		      \begin{align*}
			       & \frac{2}{1 + \lambda^{k-2}} \left( \lambda \sum_{j=1}^{(i-1)/2} \sum_{\ell=2j-1}^{k-1-2j} 2 \lambda^\ell - \sum_{j=1}^{(i-3)/2} \sum_{\ell=2j}^{k-2-2j} 2 \lambda^\ell - \sum_{\ell=i-1}^{k-1-i} \lambda^\ell \right) + \gamma \\
			       & = \frac{2}{1 + \lambda^{k-2}} \left(\sum_{j=1}^{(i-1)/2} \sum_{\ell=2j}^{k-2j} 2 \lambda^\ell - \sum_{j=1}^{(i-3)/2} \sum_{\ell=2j}^{k-2-2j} 2 \lambda^\ell - \sum_{\ell=i-1}^{k-1-i} \lambda^\ell \right) + \gamma            \\
			       & = \frac{2}{1 + \lambda^{k-2}} \left( \sum_{j=1}^{(i-1)/2} 2 (\lambda^{k-2j} + \lambda^{k-1-2j}) +  \sum_{\ell=i-1}^{k-1-i} 2 \lambda^\ell  - \sum_{\ell=i-1}^{k-1-i} \lambda^\ell \right) + \gamma                             \\
			       & = \frac{2}{1 + \lambda^{k-2}} \left( \sum_{\ell=k-i}^{k-2} 2 \lambda^{\ell} + \sum_{\ell=i-1}^{k-1-i}  \lambda^\ell \right) + \gamma.
		      \end{align*}
		      Since~$\rob \gamma = \gamma + d_1 \geq 0$, we conclude that~$\omega_i \geq 0$. Further, using the fact that~$\sum_{\ell=k-i}^{k-2} \lambda^\ell \leq \sum_{\ell=0}^{i-2} \lambda^\ell$ yields
		      \[
			      \frac{2}{1 + \lambda^{k-2}} \left( \sum_{\ell=k-i}^{k-2} 2 \lambda^{\ell} + \sum_{\ell=i-1}^{k-1-i}  \lambda^\ell \right) + \gamma \leq \frac{2}{1 + \lambda^{k-2}} \sum_{\ell=0}^{k-2} \lambda^{\ell} + \gamma = d_1 + \gamma,
		      \]
		      and we conclude that~$\omega_i \leq 1$.
\end{proof}

{Before {proving formally} %
{our robustness bound} by exhaustively reviewing all possible moves and bounding the corresponding changes of~$\Phi$, we {give} some intuition.

We choose the scaling parameter~$\gamma$ such that the decrease of~$\Phi$ exactly matches the required lower bound for the case where the request is outside the convex hull of \ldc's servers.
The remaining cases are split among the possible locations where a request can appear between two servers of \ldc, and we show in each case that~$\Phi$ decreases enough. {The definition of the $\omega$ values ensures}
that a wrong prediction gives a tight bound on the decrease of~$\Phi$ for \ldc's move, while a correct prediction still guarantees a loose bound.}

\begin{proof}[Proof of~\Cref{lemma:rob}]
	Note that~\Cref{claim:omega-bound} implies~$\Phi \geq 0$.
	Suppose that the next request arrives. First the optimal solution increases due to~\Cref{claim:omega-bound} the potential by at most~$\rob \gamma \Delta \opt$ while \ldc remains in its previous configuration. Second \ldc moves. 
	In the remaining proof we demonstrate that the potential decreases by at least~$\gamma \Delta \alg$, which proves the Lemma. We look at the following set of exhaustive cases that occur when \ldc makes its move. Assume by scaling that in each case the fast server moves distance~$1$.
	\begin{enumerate}[(a)]
		\item Let the request w.l.o.g. be on the left of~$s_1$. Hence,~$\Delta \alg = 1$, and~$\Theta$ increases by~$d_1$ due to~\Cref{claim:scale-of-theta}. Since~$x_1$ cannot be on the right side of the request, the potential changes by
		      \[
				\Delta \Phi = -\rob \gamma + d_1 = -(d_1 - \gamma) + d_1 = -\gamma \Delta \alg.
			  \]
	\end{enumerate}
	The remaining cases tackle the situations where the request is located between the two servers~$s_i$ and~$s_{i+1}$. Without loss of generality we only look at those cases where~$i \leq \floor{k/2}$, since the others hold by the symmetry of the line and by the symmetry of~$\Phi$.
	\begin{enumerate}[(a),resume]
		\item 
		Let~$1 \leq i \leq \ceil{k/2}-1$ and suppose that~$s_i$ is predicted while the optimal solution serves the request with~$x_j$ for some~$j > i$. Note that~$\Delta \alg = 1 + \lambda$. 
		The change of~$\Phi$ is at most
		      \begin{align*}
			      \Delta \Phi
			      \leq \rob \gamma (\omega_i - \lambda \omega_{i+1}) - d_i + \lambda d_{i+1}.
		      \end{align*}
		      By using the definition of~$\omega_i$ if~$i$ is odd and~\Cref{lemma:omegas-are-equal} if~$i$ is even, this is equal to
		      \begin{align*}
			       d_i - \lambda d_{i+1} - (1+ \lambda)\gamma - d_i + \lambda d_{i+1} = -\gamma (1 + \lambda) = -\gamma \Delta \alg.
		      \end{align*}
		\item 
		Let~$1 \leq i \leq \ceil{k/2}-1$ and suppose that~$s_{i+1}$ is predicted while the optimal solution serves the request with~$x_j$ for some~$j \leq i$.
		      The change of~$\Phi$ is at most
		      \begin{align*}
			      \Delta \Phi
			       \leq \rob \gamma (\omega_{i+1} - \lambda \omega_{i}) - \lambda d_i + d_{i+1}.
		      \end{align*}
		      By using the definition of~$\omega_i$ if~$i$ is even and~\Cref{lemma:omegas-are-equal} if~$i$ is odd, this is equal to
		      \begin{align*}
			       \lambda d_i - d_{i+1} - (1+ \lambda)\gamma  - \lambda d_i + d_{i+1} = -\gamma (1 + \lambda) = -\gamma \Delta \alg.
		      \end{align*}
		\item 
		Let~$1 \leq i \leq \ceil{k/2}-1$ and suppose that~$s_{i}$ is predicted while the optimal solution serves the request with~$x_j$ for some~$j \leq i$.
		      The change of~$\Phi$ is at most
		      \begin{align*}
			      \Delta \Phi
			       \leq \rob \gamma (\lambda \omega_{i+1} - \omega_{i}) - d_i + \lambda d_{i+1}.
		      \end{align*}
		      By using the definition of~$\omega_i$ if~$i$ is odd and~\Cref{lemma:omegas-are-equal} if~$i$ is even, this is equal to
		      \begin{align*}
			       & - d_i + \lambda d_{i+1} + (1+\lambda)\gamma - d_i + \lambda d_{i+1}                                           \\
			       & = -\gamma (1 + \lambda) + 2 ( \lambda d_{i+1} - d_i + (1 + \lambda) \gamma )                                  \\
			       & = -\gamma (1 + \lambda) + 2 ( \lambda d_{i+1} + \lambda^2 d_i - \lambda^2 d_i - d_i + (1 + \lambda) \gamma ).
		      \end{align*}
		      Applying~\Cref{claim:d-gamma} yields
		      \begin{align*}
			       & -\gamma (1 + \lambda) + 2 \left( (\lambda^2 - 1) d_i - (1+\lambda)\frac{\gamma}{\lambda^{i-1}} + (1 + \lambda) \gamma \right) \\
			       & \leq -\gamma (1 + \lambda)= -\gamma \Delta \alg.
		      \end{align*}
		\item 
		Let~$1 \leq i \leq \ceil{k/2}-1$ and suppose that~$s_{i+1}$ is predicted while the optimal solution serves the request with~$x_j$ for some~$j > i$.
		      The change of~$\Phi$ is at most
		      \begin{align*}
			      \Delta \Phi
			       & \leq \rob \gamma (\lambda \omega_{i} - \omega_{i+1}) - \lambda d_i + d_{i+1}.
		      \end{align*}
		      By using the definition of~$\omega_i$ if~$i$ is even and~\Cref{lemma:omegas-are-equal} if~$i$ is odd, we can conclude
		      \begin{align*}
			       & - \lambda d_i + d_{i+1} + (1 + \lambda) \gamma - \lambda d_i + d_{i+1}                                      \\
			       & = - (1 + \lambda) \gamma + 2 ( (1 + \lambda) \gamma - \lambda d_i + d_{i+1} )                             .
		      \end{align*}
		      Using~\Cref{claim:d-gamma} gives
		      \begin{align*}
			       & - (1 + \lambda) \gamma + 2 \left( (1 + \lambda) \gamma - (1 + \lambda) \frac{\gamma}{\lambda^i}  \right) \\
			       & \leq -\gamma (1 + \lambda) = -\gamma \Delta \alg.
		      \end{align*}
	\end{enumerate}
	If~$k$ is even, there are two additional cases which occur when the request is located between the two middle servers~$s_{k/2}$ and~$s_{k/2+1}$. Note that these cases cannot be covered by the previous ones, since the~$\omega$-weights of the servers on both sides of the request are equal.
	\begin{enumerate}[(a),resume]
		\item Let the request be between~$s_{k/2}$ and~$s_{k/2+1}$, and suppose that~$s_{k/2}$ is predicted while the optimal solution serves~$r$ with~$x_j$ for some~$j > k/2$. The change of~$\Phi$ is at most
		      \begin{equation}
			      \Delta \Phi
			      \leq \rob \gamma (\omega_{k/2} - \lambda \omega_{k/2}) - \lambda d_{k/2} - d_{k/2}.\label{eq:change-middle-case}
		      \end{equation}
		      For the rest of this case, we distinguish two cases according to the parity of $k/2$, and show that~$\Delta \Phi \leq -\gamma \Delta \alg$.

		      \emph{If~$k/2$ is even},~\eqref{eq:change-middle-case} is by~\Cref{lemma:omegas-are-equal} and the definition of~$\omega_{k/2}$ equal to
		      \begin{multline*}
			      2 \lambda \sum_{j=1}^{k/4} d_{2j-1} - 2 \sum_{j=1}^{k/4-1} d_{2j} - d_{k/2} - \gamma \\ - \left(2 \lambda \sum_{j=1}^{k/4-1} d_{2j} - 2 \sum_{j=1}^{k/4-1} d_{2j+1} + \lambda d_{k/2} + (2 + \lambda) \gamma\right)  - \lambda d_{k/2} - d_{k/2}.
		      \end{multline*}
		      Noting that~$2 \lambda \sum_{j=1}^{k/4} d_{2j-1} = 2 \lambda d_1 + 2\lambda \sum_{j=1}^{k/4-1} d_{2j+1}$ gives
		      \[
			      -(3 + \lambda)\gamma + 2 \left( \lambda d_1 + (1+\lambda) \sum_{j=2}^{k/2} {(-1)}^{j-1} d_{j} \right).
		      \]
		      We can conclude that this is equal to~$-\gamma(1+\lambda)$ by~\Cref{claim:gamma-alternating}.

		      Similarly,~\emph{if~$k/2$ is odd},~\eqref{eq:change-middle-case} is by~\Cref{lemma:omegas-are-equal} and the definition of~$\omega_{k/2}$ equal to
		      \begin{multline*}
			      2 \lambda \sum_{j=1}^{(k/2-1)/2} d_{2j} - 2 \sum_{j=1}^{(k/2-3)/2} d_{2j+1} - d_{k/2} + \gamma \\ - \left(2 \lambda \sum_{j=1}^{(k/2-1)/2} d_{2j-1} - 2 \sum_{j=1}^{(k/2-1)/2} d_{2j} + \lambda d_{k/2} + \lambda \gamma\right)  - \lambda d_{k/2} - d_{k/2}.
		      \end{multline*}
		      Noting that~$2 \lambda \sum_{j=1}^{(k/2-1)/2} d_{2j-1} = 2 \lambda d_1 + 2 \lambda \sum_{j=1}^{(k/2-3)/2} d_{2j+1}$ yields%
		      \[
			      \gamma - \lambda\gamma + 2 \left( -\lambda d_1 + (1+\lambda) \sum_{j=2}^{k/2} {(-1)}^{j} d_{j} \right) =
			      \gamma - \lambda\gamma - 2 \left( \lambda d_1 + (1+\lambda) \sum_{j=2}^{k/2} {(-1)}^{j-1} d_{j} \right).
		      \]
		      This is equal to~$-\gamma (1+\lambda)$ by~\Cref{claim:gamma-alternating}.
		\item Let the request be between~$s_{k/2}$ and~$s_{k/2+1}$, and suppose that~$s_{k/2}$ is predicted while the optimal solution serves~$r$ with~$x_j$ for some~$j \leq k/2$. The change of~$\Phi$ is at most
		      \begin{align*}
			      \Delta \Phi \leq \rob \gamma (\lambda \omega_{k/2} - \omega_{k/2}) - \lambda d_{k/2} - d_{k/2},
		      \end{align*}
		      which is bounded from above by the previous case. Hence,~$\Delta \Phi \leq -\gamma \Delta \alg$. \qedhere 
	\end{enumerate}
\end{proof}

\section{Proofs for \Cref{sec:tradeoff}}\label{app:tradeoff}

This section is dedicated to the proof of~\Cref{thm:pareto-k}, which we restate below. The proof is a generalization of the one proposed in~\Cref{sec:k2-tradeoff} for two servers. However, for proving the general case we need a more sophisticated construction rule and a more involved argumentation. 

\theoremTradeoffLB*

Let~$\lambda \in (0,1]$. Recall that~$\rho(k) = \sum_{i=0}^{k-1} \lambda^{i}$.
Let~$\A$ be a~$\rho(i)$-consistent \local and \memoryless deterministic online algorithm for the~$i$-server problem on the line, for all $i\leq k$. The objective is to show that \A is then at least $\beta(k)$-robust, with $\beta(k) = \sum_{i=0}^{k-1} \lambda^{-i}$.

Let~$p_1 \leq \ldots \leq p_{k+1}$ be points on the line with inter-distances~$d_1,\ldots,d_k$,  where for~$1 \leq i \leq k-1$,~$d_i = 1$, and $d_k > 1$ is arbitrarily large. See \Cref{fig:LBkI} for an illustration. We also define an {arbitrarily} small constant $\varepsilon>0$ 
and say that a server~\emph{covers} a point~$p_i$ if~it is at most a distance $\varepsilon$ away from it.
We refer to smaller positions on the line as {\it left}.
Let~$P := \{ p_1,\ldots,p_k \}$.
In the following we inductively construct an instance.
In their initial configuration,~\emph{i.e.} at time~$t=0$, the~$k$ servers, $s_1,\ldots,s_k$, are located at~$p_1,\ldots,p_k$. We assume that servers never overpass each other to simplify the notations.
 Then, we \emph{force} the servers to $p_1,\dots,p_{k-1}, p_{k+1}$ (see the \memoryless definition).
The instance terminates when \A places $s_k$ to cover $p_k$.
At any time~$t > 0$, the next requested point $r_t$ is the~\emph{leftmost} point (\emph{i.e.} the point with the smallest index) which is~\emph{not covered} by any server of \A. 
If $p_1$ is not covered and $s_1$ is on the left of $p_2-\varepsilon$ then $r_t$ is the second leftmost uncovered point. If $p_1$ is not covered and $s_1$ covers $p_2$, but did not serve it since leaving $p_1$, then $r_t$ is $p_2$ and $r_{t+1}$ (next in time) is $p_1$. 
  At any time~$t>0$, we denote the instance composed of~$r_1,\ldots,r_{t}$ by~$I_t$.

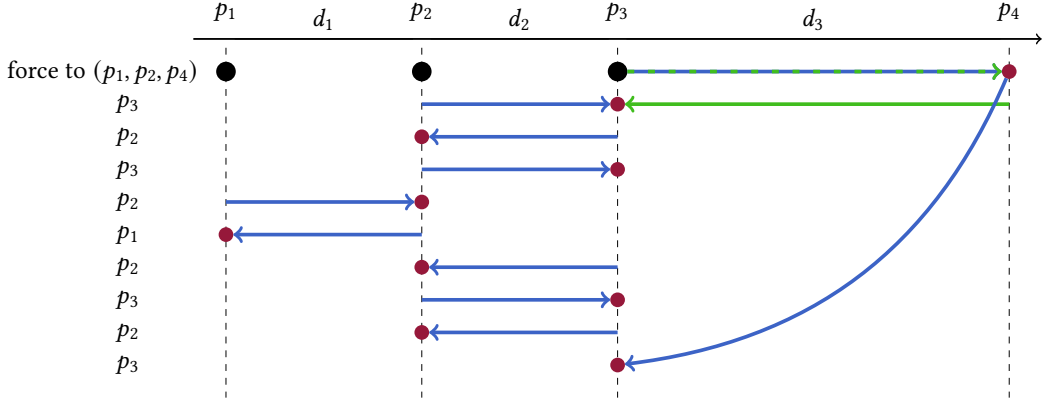
\begin{figure}[btp]
    \centering
    \begin{adjustbox}{width=0.8\linewidth}
        \begin{tikzpicture}
			\draw[thick,black,->] (0,0) -- (13,0);
			\node (a) at (0.5,0) [label=$p_1$] {};
			\node (b) at (3.5,0) [label=$p_2$] {};
			\node (c) at (6.5,0) [label=$p_3$] {};
			\node (d) at (12.5,0) [label=$p_4$] {};
			\draw[black] (a) -- node[above] {$d_1$} (b);
			\draw[black] (b) -- node[above] {$d_2$} (c);
			\draw[black] (c) -- node[above] {$d_3$} (d);
		
			\draw[helperline] (a) -- (0.5,-5.5);
			\draw[helperline] (b) -- (3.5,-5.5);
			\draw[helperline] (c) -- (6.5,-5.5);
			\draw[helperline] (d) -- (12.5,-5.5);
		
			\node (r1) at (-1, -0.5) {\hspace{-2em}force to $(p_1, p_2, p_4)$};
			\node (r2) at (-1, -1) {$p_3$};
			\node (r3) at (-1, -1.5) {$p_2$};
			\node (r4) at (-1, -2) {$p_3$};
			\node (r5) at (-1, -2.5) {$p_2$};
			\node (r6) at (-1, -3) {$p_1$};

			\node (r7) at (-1, -3.5) {$p_2$};
			\node (r8) at (-1, -4.0) {$p_3$};
			\node (r9) at (-1, -4.5) {$p_2$};
			\node (r10) at (-1, -5) {$p_3$};
		
			\node[server] (a1) at (a |- r1) {};
			\node[server] (b1) at (b |- r1) {};
			\node[server] (c1) at (c |- r1) {};
			\node[req] (req1) at (d |- r1) {};
			\node[req] (req2) at (c |- r2) {};
			\draw[algMove] (c1) -- (req1);
			\draw[predMove, loosely dashed] (c1) -- (req1);
			\draw[predMove] (d |- r2) -- (req2);
			
			\coordinate (b2) at (b |- r2) {};
			\draw[algMove] (b2) -- (req2);
			
			\node[req] (req3) at (b |- r3) {};
			\coordinate (a3) at (c |- r3) {};
			\draw[algMove] (a3) -- (req3);
			
			\node[req](req4) at (c |- r4) {};
			\coordinate (b4) at (b |- r4) {};
			\draw[algMove] (b4) -- (req4);
			
			\node[req](req5) at (b |- r5) {};
			\coordinate (a5) at (a |- r5) {};
			\draw[algMove] (a5) -- (req5);

			\node[req](req6) at (a |- r6) {};
			\coordinate (a6) at (b |- r6) {};
			\draw[algMove] (a6) -- (req6);
		
			\node[req](req7) at (b |- r7) {};
			\coordinate (a7) at (c |- r7) {};
			\draw[algMove] (a7) -- (req7);

			\node[req](req8) at (c |- r8) {};
			\coordinate (b8) at (b |- r8) {};
			\draw[algMove] (b8) -- (req8);

			\node[req](req9) at (b |- r9) {};
			\coordinate (b9) at (c |- r9) {};
			\draw[algMove] (b9) -- (req9);

			\node[req](req10) at (c |- r10) {};
			\draw[algMove] (req1) edge[bend left] (req10);
		 \end{tikzpicture}
    \end{adjustbox}
    \caption{Instance~$I$ for~$k=3$. The prediction is drawn green. The blue moves indicate an exemplary schedule of an algorithm.}\label{fig:LBkI}
\end{figure}

At every point in time, we give~$\A$ the prediction that suggests serving a request at some point~$p_i$ with the server~$s_i$. An exception is the first request, where~$s_k$ is predicted~(note that the first request is always located at~$p_{k+1}$).
We now show that this construction rule is well-defined.

\begin{lemma}\label{lemma:general-tradeoff-finite-instance}
    The construction ends after a finite number of requests.
\end{lemma}
\begin{proof}
    For the sake of contradiction, assume that the construction does not end after a finite number of requests.
    Hence, every request~$r$ except the first one must be in the set~$P$, and by construction, no server covered~$r$ in the previous configuration. 
    Thus, the server that serves~$r$ must have been moved with some cost at least $\varepsilon$, which implies that~$\A$ has unbounded cost.
    
    Now consider~\emph{any} infinite instance~$I^P$ which starts with a request at~$p_{k+1}$ followed by requests contained in~$P$. An optimal solution for~$I^P$ is to serve the first request with~$s_k$ and then to move it immediately back to the set~$P$, such that every point in~$P$ contains a server. Hence, the total cost of an optimal solution is constant. Therefore the consistency of $\A$ would be infinite, as the prediction given to~$\A$ corresponds to the optimal solution, which is a contradiction.
\end{proof}
Due to~\Cref{lemma:general-tradeoff-finite-instance}, we assume for the rest of this section that the construction ends after~$n$ steps, and we define~$I = I_n$, see \Cref{fig:LBkI}.

We first focus on the cost that~$\A$ charges for~$I$.
Let $D_i$  be the distance traveled by the server $s_i$ in \A. 
Using the \local definition, we show the following relation between~$D_i$'s:

\begin{lemma}\label{lem:D1Di}
For all $i\leq k$, for $\varepsilon$ small enough, we have $D_1 \leq \lambda^{i-1} D_i + O_k(\varepsilon D_i + d_1)$, where the notation $O_k(\cdot)$ treats $k$ as a constant.
\end{lemma}

\begin{proof}
Let $i\in \{2,3,\ldots,k\}$ and assume by induction that the relation is true for all
$j<i$. Note that it is trivial for $i=1$.

We denote by $\A_i(I)$ the cost of $\A$ restricted to the $i$ leftmost servers.
Consider the $i$ leftmost servers and we apply the \local property of \A on
these servers.   Let $I'$ be the corresponding instance on $i$ servers, {
where requests not served by $\{s_1, \dots, s_i\}$ are replaced by requests to
the new position of $s_i$.}

Consider the algorithm $\ftp$ serving $I'$ following the initial predictions as
in the \local definition.
 There are two types of requests: a point $p_\ell$ for
$\ell<i$ is served at no cost by $s_\ell$, and any other request is served by
$s_i$. The objective is to upper bound $\ftp(I')$ by $D_i$ plus negligible terms.
Consider all requests different from $p_i$ served by $s_i$ in $\ftp$,
and let $r_1$ and $r_2$ be two \emph{consecutive} requests in this set (there can be
other requests not belonging to this set between $r_1$ and $r_2$). These
requests are based on requests of $I$ outside of $\{p_1\dots p_i\}$, which means
that each of these points (except $p_1$) is covered by a server of \A before the
request, and that $s_i$ also went to $r_1$ and $r_2$ in \A, at the time at which
they are requested in $I'$. A technical difficulty here is that $s_i$ does not
need to be \emph{exactly} at $p_i$ before these requests: it can be within a
distance of $\varepsilon$. There are several cases to analyze.

\begin{itemize}
\item If $p_i$ is not requested between
$r_1$ and $r_2$, then $\ftp$ pays the shortest path between $r_1$ and $r_2$, so
at most how much $s_i$ travels in $\A$. 
\item If $s_i$ goes on $p_i$ between $r_1$ and $r_2$ in \A, then $\ftp$ also pays at most how much $s_i$ travels in $\A$.
\item If $p_i$ is requested between $r_1$ and $r_2$ and  $s_i$  does not go on $p_i$ in \A, we focus on the subinstance $I^*$ starting from the request $r_1$ and ending just before $r_2$ is requested. Let $\ftp(I^*)$, $\A_i(I^*)$ and $D_\ell^*$ be the restrictions of $\ftp(I)$, $\A_i(I)$ and $D_\ell$ to $I^*$. {Note that $\ftp(I^*) \leq D_i^* + \varepsilon$ as $\ftp$ moves $s_i$ to $r_1$ then back to {$p_i$} whereas $\A$ needs only to move $s_i$ to $r_1$ and then near {$p_i$}.} The objective is now to show that this additive term $\varepsilon$ is negligible compared to $\A_{i}(I^*)$, for which we need a further case distinction.
\begin{itemize}
\item If $r_1$ is at least a distance $\sqrt{\varepsilon}$ away from $p_i$, then
$\ftp(I^*)$ moves $s_i$ by a distance which is close to $D_i^*$. Specifically, we have $D_i^* \geq 2\sqrt{\varepsilon}-2\varepsilon \geq \sqrt{\varepsilon}$ for $\varepsilon$ small enough, and
the relation $\ftp(I^*) \leq D_i^* + \varepsilon$ implies  $\ftp(I^*) \leq (1+
\sqrt{\varepsilon}) D_i^*$.

\item If $r_1$ is at most a distance $\sqrt{\varepsilon}$ away from $p_i$, we get $\ftp(I^*) \leq 2\sqrt{\varepsilon} + \varepsilon$ and we distinguish two cases which are slightly different if $i=2$ or $i>2$.
\begin{itemize}
\item If $i>2$ then the cost of
$\A_i$ on $I^*$ is at least $A_i(I^*)\geq D_{i-1}^* >d_{i-1}-\varepsilon = 1-\varepsilon$ as $p_i$ must have
been served by $s_{i-1}$ (previously located near $p_{i-1}$) if it was not served by $s_i$.  We therefore obtain $\ftp(I^*) \leq 3\sqrt{\varepsilon} \cdot D_{i-1}^*$.

\item If $i=2$, the difference is that {$s_1$} may be initially located anywhere between $p_1$ and $p_2$. $s_1$ serves $p_i=p_2$ when it is requested (as this case assumes $s_2$ does not serve $p_2$ in $I^*$), and then must serve $p_1$ by the definition of the instance $I$. 
Therefore, the cost of
$\A_i$ on $I^*$ is at least $A_i(I^*)\geq D_1^*\geq d_1=1$. We thus obtain $\ftp(I^*) \leq 3{\sqrt{\varepsilon}}\cdot D_{1}^*$. 
\end{itemize}
\end{itemize} 
\end{itemize}

Summing over all subinstances, we obtain the following inequality:
\[
	\ftp(I') \leq (1+\sqrt\varepsilon) D_i + 3\sqrt{\varepsilon} \cdot \sum_{\ell=1}^{i-1} D_\ell \leq D_i + 3\sqrt \varepsilon\cdot \A_{i}(I).
\]

As the initial and final configurations are identical up to a distance of $d_1$ for $s_1$ and $\varepsilon$ for other servers, the \local property for $I'$ yields
\begin{align*}
\A_{i}(I) &\leq \rho(i) D_i + 3\sqrt{\varepsilon}\rho(i) \cdot  \A_{i}(I) + O(\varepsilon k^2 + d_1k).
\end{align*}
 
{
For $\varepsilon$ small enough, we have $3\sqrt{\varepsilon}\rho(i)<1/2$,  which implies that $\A_{i}(I) \leq 2\rho(i)D_i + O(\varepsilon k^2+ d_1k)$. Using this new bound on $\A_i(I)$ on the right-hand side of the above inequality leads to the following:}
\begin{align*}
\A_{i}(I) &\leq \rho(i) D_i+ O(\varepsilon k^2 + d_1k+ \sqrt\varepsilon\rho(i)^2 D_i)\\
 \sum_{\ell=1}^i D_\ell &\leq  D_i + (\rho(i)-1) D_i + O(\varepsilon k^2 + d_1k+ \sqrt\varepsilon\rho(i)^2 D_i)\\
 \sum_{\ell=1}^{i-1} D_\ell &\leq  (\rho(i)-1) D_i + O(\varepsilon k^2 + d_1k+ \sqrt\varepsilon\rho(i)^2 D_i)
\end{align*}

{
We now use the induction hypothesis to lower bound $D_\ell$ by ${\lambda^{1-\ell}} D_1+O_k(\varepsilon D_i+ d_1)$ and replace $\rho(i)$ by its expression, before dividing all sides by $\sum_{\ell=0}^{i-2}\lambda^{-\ell}$. We use the notation $O_k(\cdot)$ to avoid detailing the irrelevant dependencies on $k$, note that $\rho(i)$ depends only on $\lambda$ and $k$ so does not appear inside the notation $O_k(\cdot)$.
}

\begin{align*}
\sum_{\ell=1}^{i-1} \frac{1}{\lambda^{\ell-1}} D_1 &\leq  \sum_{\ell=1}^{i-1} \lambda^\ell D_i + O_k(\varepsilon  D_i+ d_1)\\
D_1 &\leq \lambda^{i-1} D_i + O_k(\varepsilon D_i+ d_1)
\end{align*}
\end{proof}

We build the instance $I^\omega$ repeating the instance $I$ $\omega$ times, starting directly by the force to $p_1,\dots,p_{k-1},p_{k+1}$, see \Cref{fig:LBkIw}.  The predictions for each iteration correspond to the predictions defined in instance~$I$. We now bound the optimal cost for this instance.

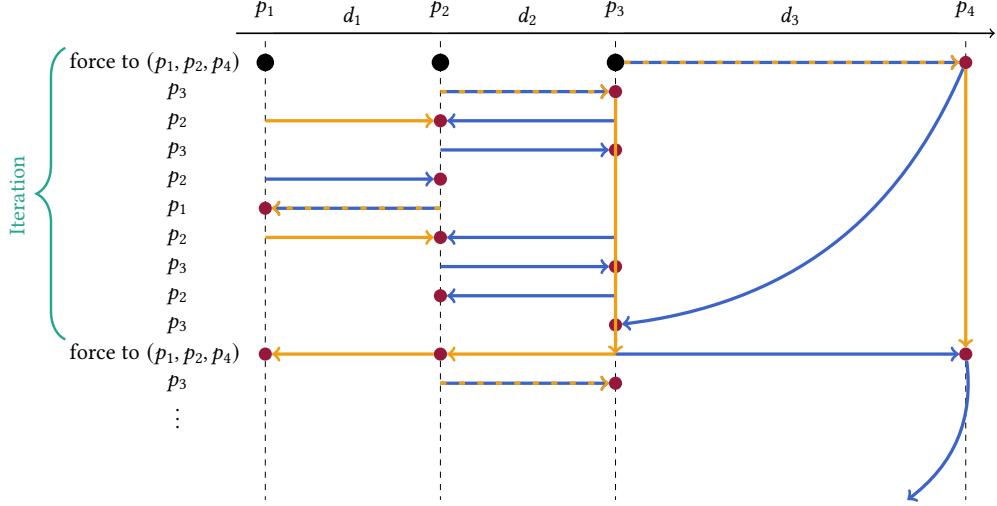
\begin{figure}[tb]
    \centering
    \begin{adjustbox}{width=0.8\linewidth}
        \begin{tikzpicture}
			\draw[thick,black,->] (0,0) -- (13,0);
			\node (a) at (0.5,0) [label=$p_1$] {};
			\node (b) at (3.5,0) [label=$p_2$] {};
			\node (c) at (6.5,0) [label=$p_3$] {};
			\node (d) at (12.5,0) [label=$p_4$] {};
			\draw[black] (a) -- node[above] {$d_1$} (b);
			\draw[black] (b) -- node[above] {$d_2$} (c);
			\draw[black] (c) -- node[above] {$d_3$} (d);
		
			\draw[helperline] (a) -- (0.5,-8);
			\draw[helperline] (b) -- (3.5,-8);
			\draw[helperline] (c) -- (6.5,-8);
			\draw[helperline] (d) -- (12.5,-8);
		
			\node (r1) at (-1, -0.5) {\hspace{-2em}force to $(p_1, p_2, p_4)$};
			\node (r2) at (-1, -1) {$p_3$};
			\node (r3) at (-1, -1.5) {$p_2$};
			\node (r4) at (-1, -2) {$p_3$};
			\node (r5) at (-1, -2.5) {$p_2$};
			\node (r6) at (-1, -3) {$p_1$};

			\node (r7) at (-1, -3.5) {$p_2$};
			\node (r8) at (-1, -4.0) {$p_3$};
			\node (r9) at (-1, -4.5) {$p_2$};
			\node (r10) at (-1, -5) {$p_3$};
		
			\node[server] (a1) at (a |- r1) {};
			\node[server] (b1) at (b |- r1) {};
			\node[server] (c1) at (c |- r1) {};
			\node[req] (req1) at (d |- r1) {};
			\node[req] (req2) at (c |- r2) {};
			\draw[algMove] (c1) -- (req1);
			\draw[optMove, loosely dashed] (c1) -- (req1);

			\coordinate (b2) at (b |- r2) {};
			\draw[algMove] (b2) -- (req2);
			
			\node[req] (req3) at (b |- r3) {};
			\coordinate (a3) at (c |- r3) {};
			\draw[algMove] (a3) -- (req3);
			
			\node[req](req4) at (c |- r4) {};
			\coordinate (b4) at (b |- r4) {};
			\draw[algMove] (b4) -- (req4);
			
			\node[req](req5) at (b |- r5) {};
			\coordinate (a5) at (a |- r5) {};
			\draw[algMove] (a5) -- (req5);

			\node[req](req6) at (a |- r6) {};
			\coordinate (a6) at (b |- r6) {};
			\draw[algMove] (a6) -- (req6);
		
			\node[req](req7) at (b |- r7) {};
			\coordinate (a7) at (c |- r7) {};
			\draw[algMove] (a7) -- (req7);

			\node[req](req8) at (c |- r8) {};
			\coordinate (b8) at (b |- r8) {};
			\draw[algMove] (b8) -- (req8);

			\node[req](req9) at (b |- r9) {};
			\coordinate (b9) at (c |- r9) {};
			\draw[algMove] (b9) -- (req9);

			\node[req](req10) at (c |- r10) {};
			\draw[algMove] (req1) edge[bend left] (req10);

			\draw[optMove, loosely dashed] (b2) -- (req2);
			
			\draw[optMove] (a |- r3) -- (req3);
			
			\draw[optMove, loosely dashed] (a6) -- (req6);
		
			\draw[optMove] (a |- r7) -- (req7);

			\node (r17) at (-1, -5.5) {\hspace{-2em}force to $(p_1, p_2, p_4)$};
			\node (r18) at (-1, -6) {$p_3$};
			\node (r19) at (-1, -6.5) {$\vdots$};
		
			\node[req](req17d) at (d |- r17) {};
			\node[req](req17a) at (a |- r17) {};
			\node[req](req17b) at (b |- r17) {};
			\coordinate (c17) at (c |- r17) {};
			\draw[algMove] (c17) -- (req17d);
			\draw[algMove] (req17d) edge[bend left] (11.5, -8);
			
			\coordinate (b18) at (b |- r18) {};
			\node[req](req18) at (c |- r18) {};
			\draw[algMove] (b18) -- (req18);
			\draw[optMove, loosely dashed] (b18) -- (req18);
			\draw[optMove] (req17b) -- (req17a);
			\draw[optMove] (c |- r17) -- (req17b);
		
			\draw[optMove] (req1) -- (req17d);
			\draw[optMove] (req2) -- (c |- r17);

			\draw [decorate,decoration={brace,amplitude=14pt, mirror, raise=5pt},yshift=0pt,very thick, algGreen] (-2.75,-0.25) -- (-2.75, -5.25) node [midway,left,yshift=.8cm, xshift=-1cm,rotate=90] {Iteration}; 
		 \end{tikzpicture}
    \end{adjustbox}
    \caption{Instance~$I^\omega$ for~$k=3$. The alternative (better) solution is drawn orange. The prediction and the exemplary moves of the algorithm are the same as in instance~$I$ for each iteration.}\label{fig:LBkIw}
\end{figure}

\begin{lemma}\label{lemma:opt-i-prime}
    $\opt(I^\omega) \leq 2d_k + \omega \cdot (D_1 + 2\sum_{i=2}^{k-1} d_i)$.
\end{lemma}
\begin{proof}
Consider the following schedule for $I^\omega$: at each iteration, move $k-1$ servers
to $p_2,\dots,p_{k+1}$ and alternate between $p_1$ and $p_2$ with $s_1$. We now
analyze how many alternations we need to do in each iteration. By definition of
the instance, $p_1$ is only requested if $s_1$ has served $p_2$ since it last
left $p_1$. Therefore, the distance traveled by $s_1$ equals $D_1$.
At the end of the iteration, we move back the $k-2$ middle servers, giving the target cost.
\end{proof}

\begin{proof}[Proof of~\Cref{thm:pareto-k}]
 As $\A$ is \memoryless, its behavior on each iteration of $I$ is
identical, $s_k$ is at $p_k$ initially,
then the $k$ servers are forced to the points $p_1,\dots,p_{k-1},p_{k+1}$ before
continuing the requests. Therefore $\A$ must pay at least $d_k$ to serve the
first force operation, and then must make the same decisions in all iterations.

Using~\Cref{lemma:opt-i-prime}, the competitive ratio of~$\A$ for instance~$I^\omega$ is therefore at least
    \begin{align*}
        \frac{\A(I^\omega)}{\opt(I^\omega)} &\geq  \frac{\omega \cdot \sum_{i=1}^k D_i}{2d_k + \omega \cdot (D_1 + 2\sum_{i=2}^{k-1} d_i)}\\
        &\xrightarrow{\omega \to \infty}  \frac{\sum_{i=1}^k D_i}{D_1 + 2\sum_{i=2}^{k-1} d_i}.
\end{align*}

Consider $d_k$ arbitrarily large (but still small compared to $\omega$). If $D_1$ is bounded by a constant, then the competitive ratio is unbounded, so $\A$ is not robust. Otherwise, the terms $d_i$ become negligible compared to $D_1$, and we show that the limit of the competitive ratio is lower bounded by the desired robustness expression, using \Cref{lem:D1Di} {(which implies that $d_1$ is also negligible compared to any $D_i$):}

\begin{align*}
      \frac{\A(I^\omega)}{\opt(I^\omega)}   &\xrightarrow{d_k \to \infty} \frac{\sum_{i=1}^k D_i}{D_1} 
	  \geq  \sum_{i=0}^{k-1} \frac 1{\lambda^{i}+O_k(\varepsilon + \frac{d_1}{{D_{i+1}}})}
	  \xrightarrow{\varepsilon \to 0} 
 	  \sum_{i=0}^{k-1} \lambda^{-i}.\qedhere 
\end{align*}
\end{proof}

{In the following we show that the consistency of \ldc is best possible up to a factor of $2$.}

\begin{lemma}\label{lemma:consistencies}
    For every~$\lambda \in [0,1]$,~$\alpha(k) < 2 \rho(k)$.
\end{lemma}
\begin{proof}%
    First note that for~$\lambda = 1$,~$\alpha(k) = k = \rho(k)$. Now suppose that~$\lambda < 1$. Applying the formula for the finite geometric series gives
    \[
        \rho(k) = \frac{1 - \lambda^k}{1 - \lambda}.
    \]
	We now prove the result based on the parity of $k$.  Assume that~$k$ is even. Recall that
              \[
                  \alpha(k) = 1 + 2\sum_{i=1}^{k/2-1} \lambda^i + \lambda^{k/2} = 1 + 2 \frac{\lambda - \lambda^{k/2}}{1-\lambda} + \lambda^{k/2}
              \]
              and, thus, 
              \begin{align*}
                  \frac{\alpha(k)}{\rho(k)} = \frac{(1 + \lambda^{k/2})(1-\lambda) + 2(\lambda - \lambda^{k/2})}{1 - \lambda^k} = \frac{1 + \lambda - \lambda^{k/2} - \lambda^{k/2+1}}{1 - \lambda^k} < 2.
              \end{align*}
         Assume that~$k$ is odd, then
              \[
                  \alpha(k) = 1 + 2\sum_{i=1}^{(k-1)/2} \lambda^i =  1 + 2 \frac{\lambda - \lambda^{(k+1)/2}}{1-\lambda},
              \]
              and we conclude that
              \begin{align*}
                  \frac{\alpha(k)}{\rho(k)} = \frac{1-\lambda + 2(\lambda - \lambda^{(k+1)/2})}{1 - \lambda^k} = \frac{1 + \lambda - 2 \lambda^{(k+1)/2}}{1 - \lambda^k} < 2.
              \end{align*}
\end{proof}

\section{PAC Learnability of Predictions}\label{app:learnability}

We show that our predictions are PAC learnable in an agnostic sense with a sample complexity polynomial in the number of requests and we give an efficient learning algorithm.
Let~$\D$ be an unknown distribution of sequences of~$n$ requests represented by points in the interval $[0,1]$. Here we assume a bounded line as a metric (scaled to $[0,1]$), which is a restriction but natural in most applications. Further, we assume that we can sample i.i.d.\ sequences~from~$\D$. 

Let $\Hyp = \{1,\ldots,k\}^n$ denote a hypothesis class containing all possible static predictions, i.e., the set of all $k$-server solutions for request sequences of length $n$. Let $C_0$ be a known initial configuration. The prediction error for a prediction~$\pred \in \Hyp$ on a request sequence~$\seq$ is defined as~$\eta_\seq(\pred) = \ftp(\seq,\pred) - \opt(\seq)$, where $\ftp(\seq,\pred)$ is the total cost of following the prediction $\pred$ on the sequence~$\seq$ starting in~$C_0$, and~$\opt(\seq)$ is the cost of an optimal solution on~$\seq$ starting in~$C_0$. 
Then,~$\eta_\seq(\pred) \leq \eta_{\max} \leq n$ %
for all possible sequences~$\seq$ and~for~all~$\pred \in \Hyp$.

We argue that we can use a classical \emph{empirical risk minimization (ERM)} learning method, see, e.g.,~\cite{shalevB14ML}. The ERM method uses a training set~$S=\{\seq_1,\ldots,\seq_m\}$ of i.i.d.\ samples from~$\D$.
Then, it determines a prediction~$\pred_p \in \Hyp$ that minimizes the \emph{empirical error} $\eta_S(\pred) = \frac{1}{m} \sum_{j=1}^m \eta_{\seq_j}(\pred)$. Since our hypothesis class is finite and the error function bounded, classical results imply that our predictions are PAC learnable in an agnostic sense with a polynomial sample complexity. Further, we show that the problem of finding the prediction minimizing the empirical error within the training set can be reduced to an offline $k$-server problem on a modified request sequence~$\widetilde{\seq}$ of length~$n$, where the distance between the~$\ell$th and~$i$th request in~$\widetilde{\seq}$ is given by~$\frac{1}{m} \sum_{j=1}^m d(\seq_j(\ell), \seq_j(i))$. This problem can be solved efficiently~\cite{chrobak1991dc}. %

\ThmLearnability*

\begin{proof}
	Since the hypothesis class $\Hyp$ is finite with~$|\Hyp| = k^n$, and our non-negative error function is bounded by~$\eta_{\max}$, 
	classical results, see e.g.~\cite{shalevB14ML}, imply that~$\Hyp$
	is agnostically PAC-learnable using the ERM algorithm with a sample complexity of
	\[
		m \leq  \left\lceil \frac{2 \log(2 |\Hyp| / \delta)\eta_{\max}^2}{\epsilon^2}\right\rceil \in \bigO \left( \frac{(n \log k - \log \delta)\eta_{\max}^2}{\epsilon^2} \right).
	\]
	That is, given a sample of size at least~$m$, the ERM algorithm outputs with a probability of at least~$(1-\delta)$ a prediction~$\pred_p$ such that~$\E_{\seq \sim \D}[\eta_\seq(\pred_p)] \leq \E_{\seq \sim \D}[\eta_{\seq}(\pred^*)] + \epsilon$ holds, where~$\pred^* = \arg \min_{\pred \in \Hyp} \E_{\seq \sim \D}[\eta_\seq(\pred)]$.

	It remains to describe an efficient implementation of the ERM algorithm for our setting. Let~$S = \{\seq_1,\ldots,\seq_m\}$ be a sample drawn i.i.d. from~$\D$. We assume that this can be done in polynomial time in~$m$. For a sequence~$\sigma_j \in S$ let~$\sigma_j(i)$ be the position of the $i$th request in~$\sigma_j$. We further define for~$1 \leq \ell \leq i \leq n$ and~$1 \leq k' \leq k$ the distance functions~$\delta_j(\ell, i) = \abs{\seq_j(\ell) - \seq_j(i)}$ and $\gamma_j(k', i) = \abs{C_0(k') - \seq_j(i)}$. The empirical error of a prediction~$\pred$ is in our setting defined as
	\[
		\eta_S(\pred) = \frac{1}{m} \sum_{j=1}^m \eta_{\seq_j}(\pred) = \frac{1}{m} \sum_{j=1}^m \ftp(\seq_j,\pred) - \opt(\seq_j).
	\]
	The ERM algorithm outputs the prediction~$\pred_p \in \Hyp$ that minimizes~$\eta_S(\pred)$ as a function over~$\Hyp$. Since iterating over all predictions in~$\Hyp$ takes exponential time, we compute~$\pred_p$ differently. To do so, we first observe that $\frac{1}{m} \sum_{j=1}^m \opt(\seq_j)$ is independent of~$\pred$, thus minimizing~$\eta_S(\pred)$ can be reduced to minimizing
	\begin{align}
		\frac{1}{m} \sum_{j=1}^m \ftp(\seq_j,\pred) \nonumber
		&= \frac{1}{m} \sum_{j=1}^m \sum_{k'=1}^k \sum_{i=1}^n \xi^\pred_{k',i} \cdot \gamma_j(k',i) + \sum_{\ell = 1}^i \chi^\pred_{k',i,\ell} \cdot \delta_j(\ell, i)  \nonumber\\
		&= \sum_{k'=1}^k \sum_{i=1}^n  \xi^\pred_{k',i} \cdot \left( \frac{1}{m} \sum_{j=1}^m \gamma_j(k',i) \right) + \sum_{\ell = 1}^i \chi^\pred_{k',i,\ell} \cdot \frac{1}{m} \sum_{j=1}^m \delta_j(\ell, i),\label{eq:pac1}
	\end{align}
	where $\chi^\pred_{k',i,\ell} \in \{0,1\}$ indicates (i.e. is equal to~$1$) that server~$k'$ serves the $i$th request of~$\seq_j$ directly after the~$\ell$th request of~$\seq_j$ in $\pred$ and~$\xi^\pred_{k',i} \in \{0,1\}$ indicates that the $i$th request of~$\seq_j$ is the first one that server~$k'$ serves in~$\pred$.
	
	We now demonstrate that we can efficiently compute a prediction~$\pred \in \Hyp$ that minimizes~\eqref{eq:pac1}. 
	Indeed, observe that~\eqref{eq:pac1} is equal to the total cost of the solution~$\pred$ for the $k$-server instance that starts in~$C_0$ and serves a sequence~$\widetilde{\seq}$ of length~$n$, where the distance between the~$\ell$th and~$i$th request in~$\widetilde{\seq}$ is given by~$\delta'(\ell, i) = \frac{1}{m} \sum_{j=1}^m \delta_j(\ell, i)$ and the distance between the~$i$th request in~$\widetilde{\seq}$ and the initial position of server~$k'$ is given by $\gamma'(k', i) = \frac{1}{m} \sum_{j=1}^m \gamma_j(k',i)$. But this means that any optimal solution~$\widetilde{\pred}$ for this instance also minimizes~\eqref{eq:pac1}. Clearly,~$\widetilde{\pred} \in \Hyp$, and an optimal solution for a k-server instance with known distance functions can be computed in~$\bigO(kn^2)$ time using a min-cost flow algorithm~\cite{chrobak1991dc}.
\end{proof}

\section{Experiments}\label{app:experiments}

This section gives a detailed overview over the empirical experiments.
The simulation software is written in Rust (version 1.51.0, 2018 edition). We executed all experiments in Ubuntu 18.04.5 on a machine with two AMD EPYC ROME 7542 CPUs (64 cores in total) and 1.96 TB RAM.

We implemented \ftpdc~\cite{antoniadis2020mts} with the hyperparameter~$\gamma$ equal to~$1$.
The instances are based on the BrightKite-Dataset~\cite{cho2011friendship}.
We extract sequences with a length of $1000$ checkins, normalize the scaling of latitudes to the interval $[0,4000]$, and use these values as the positions of the requests on the line. All servers start at the same initial random position.

We generate predictions in a semi-random fashion. 
Fix two parameters~$p$, the number of \emph{bins}, and~$b$, the \emph{bin size}, and an instance. 
Our goal is to generate evenly distributed predictions,~i.e., in each bin~$i \in \{1,\ldots,p\}$ {there} are at least five predictions with relative error between~$(i-1)b$ and~$ib$. 
Additionally, we use an optimal solution of the instance as the perfect prediction. 

Given those parameters and an instance, we iteratively sample many predictions with an increasing number of wrong choices with respect to the optimal solution. While this procedure does not find all predictions, especially these with the largest relative error, it gives a good tradeoff between running time and range of prediction error. 
We set~$p=10$ and~$b$ as high as we find for at least~$40$ instances these evenly distributed predictions.
Other instances are discarded.

The results for~$k=2$,~$k=10$ and~$k=50$ are displayed in \Cref{fig:experimental-results-k2,fig:experimental-results-k10,fig:experimental-results-k50}.
We first observe that for a reasonable choice of~$\lambda$ ($0.1 \leq \lambda \leq 0.5$) \ldc outperforms \ftpdc throughout almost all generated relative prediction errors in both lazy and non-lazy settings. This is also the case compared to \dc with the exception of its strong performance for $k=2$ in the lazy implementation.
Further, all algorithms except \ldc for~$\lambda = 0.0$ improve by a lazy implementation. This is no surprise, as this is the only algorithm that only moves a single server in the non-lazy setting, so there are no postponed moves that can possibly be improved by a lazy implementation. The actual improvements of the largest mean empirical ratio for any bin of all algorithms which we discovered in our experiments are given in~\Cref{table:lazy}. Observe that \ldc benefits more from the lazy implementation when~$\lambda$ gets closer to 1, whereas the improvements for \ftpdc are between 24\% and 28\%. We suspect that \ftpdc only makes few expensive resets in our instances, while \ldc benefits from many cheap improvements. 

\begin{table}
	\centering
	\begin{tabular}{lcccccccc}
		\toprule
		& \multicolumn{2}{c}{\dc} & \multicolumn{2}{c}{\ldc (0.1)} & \multicolumn{2}{c}{\ldc (0.5)} & \multicolumn{2}{c}{\ftpdc} \\ \cmidrule(lr){2-3} \cmidrule(lr){4-5} \cmidrule(lr){6-7} \cmidrule(lr){8-9}
		& non-lazy & lazy & non-lazy & lazy & non-lazy & lazy & non-lazy & lazy  \\
		\midrule
		$k=2$ & 1.60 & 1.03 & 1.31 & 1.10 & 1.35 & 1.035 & 1.70 & 1.23 \\
		Improvement & \multicolumn{2}{c}{35\%} & \multicolumn{2}{c}{16\%} & \multicolumn{2}{c}{22\%} & \multicolumn{2}{c}{27\%} \\
		\midrule
		$k=10$ & 1.63 & 1.29 & 1.35 & 1.19 & 1.43 & 1.16 & 2.14 & 1.63 \\
		Improvement & \multicolumn{2}{c}{21\%} & \multicolumn{2}{c}{12\%} & \multicolumn{2}{c}{19\%} & \multicolumn{2}{c}{24\%} \\
		\midrule
		$k=50$ & 1.63 & 1.29 & 1.45 & 1.27 & 1.44 & 1.17 & 2.29 & 1.66 \\
		Improvement & \multicolumn{2}{c}{21\%} & \multicolumn{2}{c}{12\%} & \multicolumn{2}{c}{19\%} & \multicolumn{2}{c}{28\%} \\
		\bottomrule
	\end{tabular}
	\caption{Relative improvements of the largest mean empirical competitive ratio for any bin due to lazy implementations.}\label{table:lazy}
\end{table}

\begin{figure}[tb]
    \begin{subfigure}[t]{0.49\textwidth}
        \includegraphics[width=\textwidth]{bk_k2.pdf}
        \caption{non-lazy}
    \end{subfigure}\hfill
	\begin{subfigure}[t]{0.49\textwidth}
        \includegraphics[width=\textwidth]{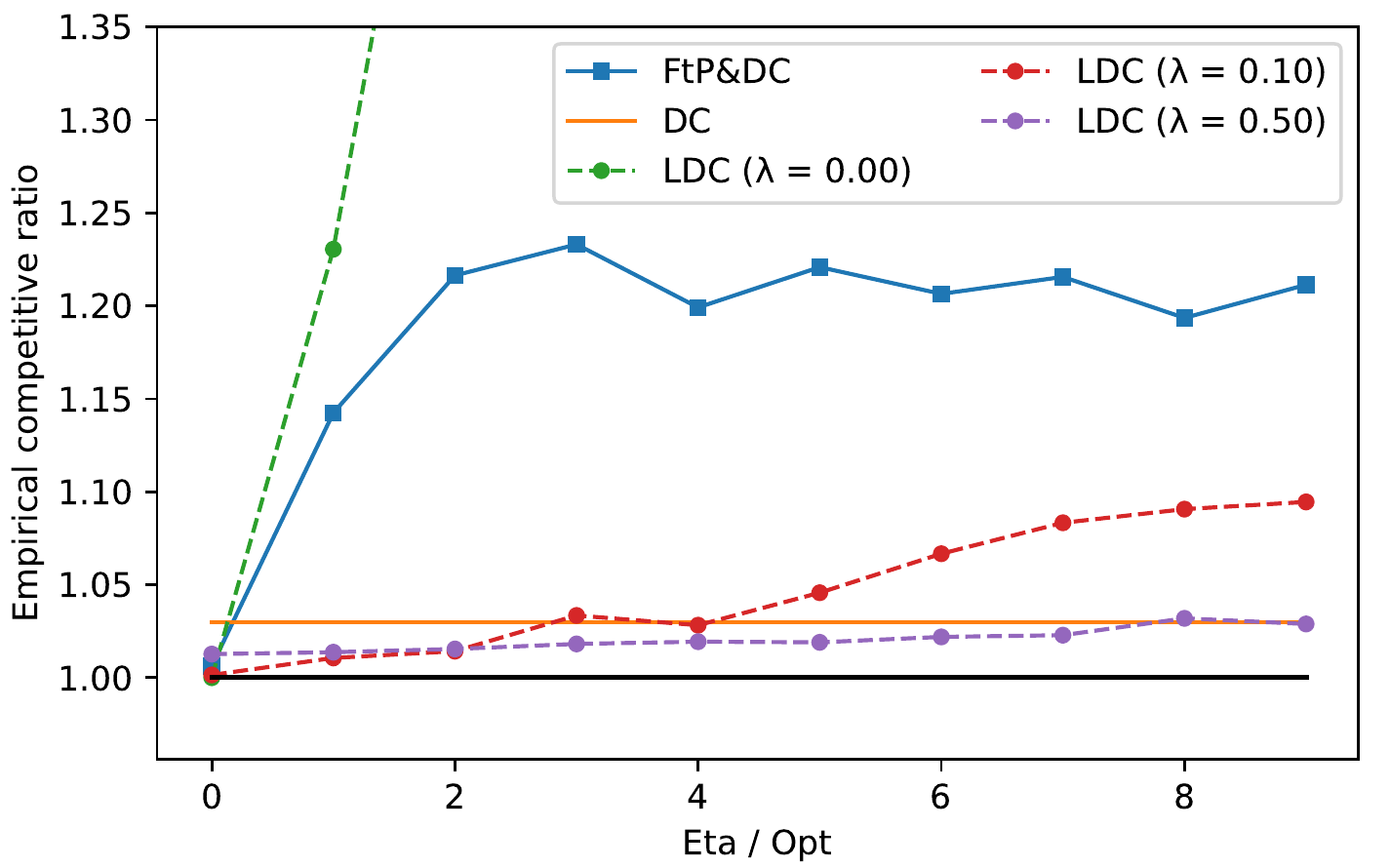}
        \caption{lazy}
    \end{subfigure}   
    \caption{Results for $k=2$ and~$b=1$.}\label{fig:experimental-results-k2}
\end{figure}

\begin{figure}[tb]
    \begin{subfigure}[t]{0.49\textwidth}
        \includegraphics[width=\textwidth]{bk_k10.pdf}
        \caption{non-lazy}
    \end{subfigure}\hfill
    \begin{subfigure}[t]{0.49\textwidth}
        \includegraphics[width=\textwidth]{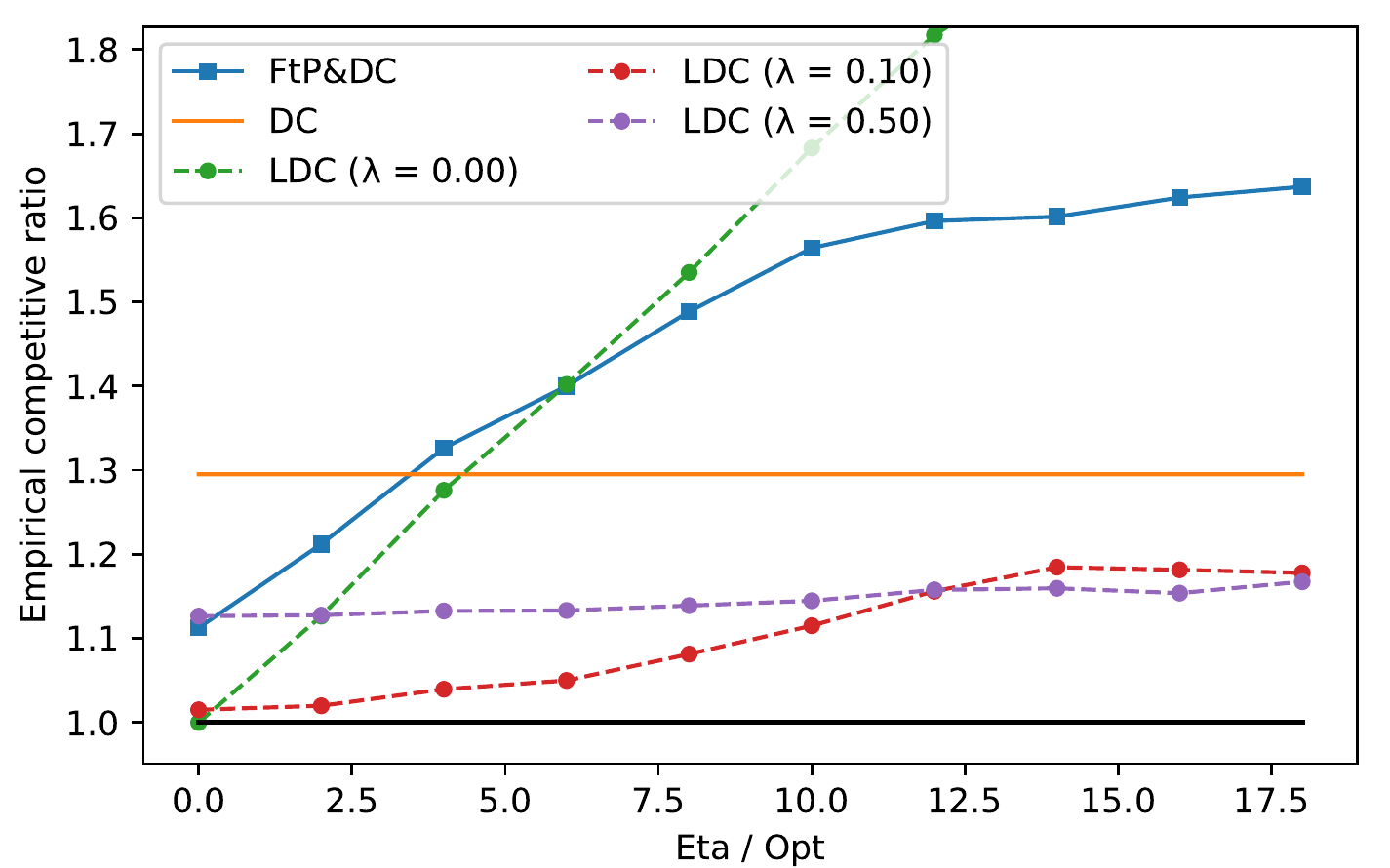}
        \caption{lazy}
    \end{subfigure}
    \caption{Results for $k=10$ and~$b=2$.}\label{fig:experimental-results-k10}
\end{figure}

\begin{figure}[tb]
    \begin{subfigure}[t]{0.49\textwidth}
        \includegraphics[width=\textwidth]{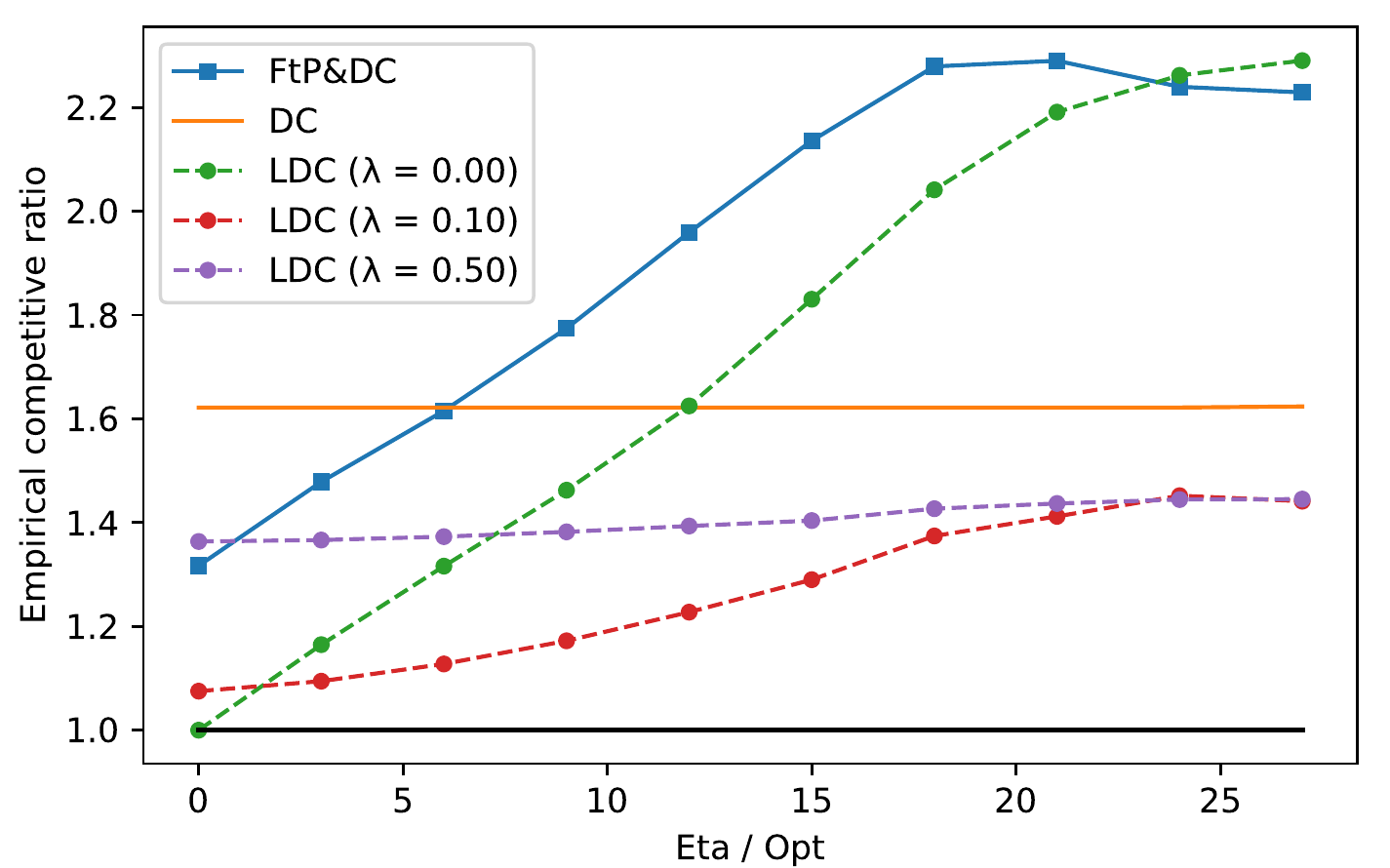}
        \caption{non-lazy}
    \end{subfigure}\hfill
    \begin{subfigure}[t]{0.49\textwidth}
        \includegraphics[width=\textwidth]{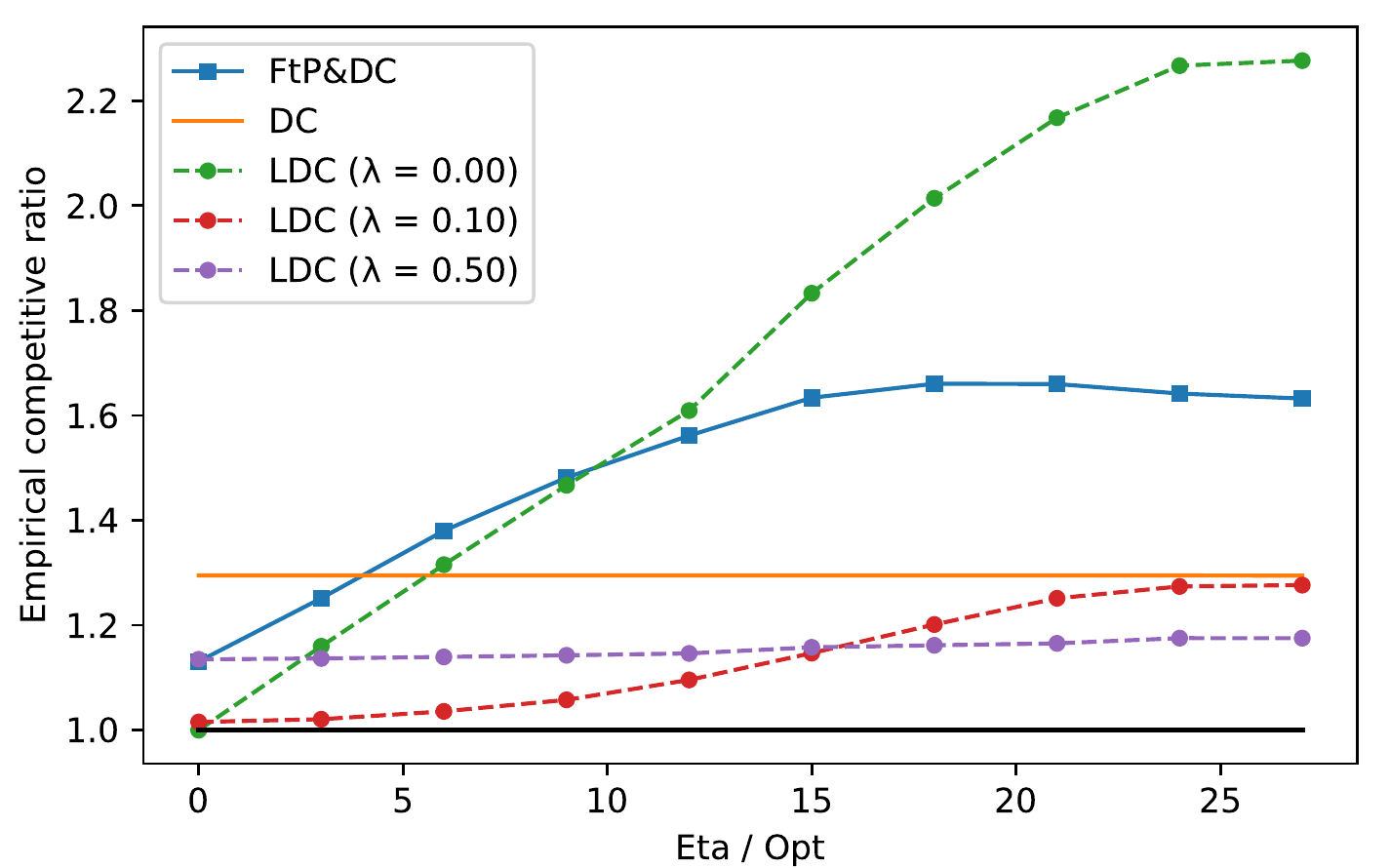}
        \caption{lazy}
    \end{subfigure}
    \caption{Results for $k=50$ and~$b=3$.}\label{fig:experimental-results-k50}
\end{figure}

\end{document}